\documentclass[12pt]{article}

\usepackage{amsmath,amsthm,amssymb,amsopn,amsfonts,pdfpages,setspace}
\usepackage{algpseudocode, algorithm}
 \usepackage{mathtools}
 \usepackage{nccmath}
\usepackage{graphics,relsize}
\usepackage{graphicx}
\usepackage{subfig,dsfont}
\usepackage{bbm}
\usepackage{caption}
\newcommand{\GM}{G_{G_2\rightarrow G_1}}
\newcommand{\GMs}{G_{\boldsymbol{\sigma}(G_2)\rightarrow G_1}}

\newcommand{\fo}{f^{\tau}}
\newcommand{\eo}{\epsilon^{\tau}}

\newcommand{\sG}{\mathcal{G}}
\newcommand{\po}{P_{\tau}}
\newcommand{\pt}{P_{\tau}}

\newcommand{\p}{\mathbb{P}}
\newcommand{\ps}{\mathbb{P}_{\boldsymbol{\sigma}} }
\newcommand{\e}{\mathbb{E}}

\newtheorem{theorem}{Theorem}
\newtheorem{claim}[theorem]{Claim}
\newtheorem{proposition}[theorem]{Proposition}
\newtheorem{lemma}[theorem]{Lemma}
\newtheorem{defn}[theorem]{Definition}

\newtheorem{remark}[theorem]{Remark}
\newtheorem{conjecture}[theorem]{Conjecture}
\makeatother
\begin{document}
\title{Information Recovery in Shuffled Graphs\\ via Graph Matching}
\author{Vince Lyzinski\\
\small{University of Massachusetts Amherst}} 

\maketitle

\begin{abstract}
While many multiple graph inference methodologies operate under the implicit assumption that an explicit vertex correspondence is known across the vertex sets of the graphs, in practice these correspondences may only be partially or errorfully known. 
Herein, we provide an information theoretic foundation for understanding the practical impact that errorfully observed vertex correspondences can have on subsequent inference, and the capacity of graph matching methods to recover the lost vertex alignment and inferential performance.
%
Working in the correlated stochastic blockmodel setting, 
we establish a duality between the loss of mutual information due to an errorfully observed vertex correspondence and the ability of graph matching algorithms to recover the true correspondence across graphs.
In the process, we establish a phase transition for graph matchability in terms of the correlation across graphs, and we conjecture the analogous phase
transition for the relative information loss due to shuffling vertex labels.
We demonstrate the practical effect that graph shuffling---and matching---can have on subsequent inference, with examples from two sample graph hypothesis testing and joint spectral graph clustering.

\end{abstract}
\maketitle

\section{Introduction}
Graphs are an increasingly popular data modality in scientific research and statistical inference, with diverse applications in connectomics \cite{bullmore2009complex}, social network analysis \cite{carrington2005models}, and pattern recognition \cite{kandel2007applied}, to name a few.  
Many joint graph inference methodologies (see, for example, \cite{MT2,gray2012magnetic,bullmore2009complex,richiardi2011decoding}), joint graph embedding algorithms (see, for example, \cite{MMCV2,JOFC,sunpriebe2013,shen2014manifold}) and graph-valued time-series methodologies (see, for example, \cite{NL2,priebeTS,tangTS,wang2014locality}) operate under the implicit assumption that an explicit vertex correspondence is {\it a priori} known across the vertex sets of the graphs. 
While this assumption is natural in a host of real data settings, in many applications these correspondences may be unobserved and/or errorfully observed \cite{vogelstein2011shuffled}.
Connectomics offers a striking example of this continuum.  
Indeed, while for some simple organisms (e.g., the C. elegans roundworm \cite{white1986structure}) explicit neuron labels are known across specimen, and in human DTMRI connectomes, the vertices are often regions of the brain registered to a common template (see \cite{gray2012magnetic}), explicit cross-subject neuron labels are often unknown for more complex organisms.

How can we quantify the effect of the added uncertainty due to an errorfully observed vertex correspondence? 
Heuristically, if $(G_1,G_2)$ is a realization from a bivariate random graph model with the property that vertices that are aligned across graphs behave similarly in their respective networks, then the uncertainty in $G_2$ is greatly reduced by observing $G_1$ and the latent alignment.  
Indeed, in the extreme case of $G_1$ and $G_2$ being isomorphic, observing the latent alignment function and $G_1$ completely determines $G_2$.
However, as the vertex labels are shuffled uncertainty is introduced into the bivariate model.
In order to formalize this heuristic, we adopt an information theoretic perspective (see \cite{cover2012elements} for the necessary background).
We develop a bivariate graph model, the $\rho$-correlated stochastic blockmodel (Section \ref{S:rhocorr}), in which we are able to
formally address the information loss/increase in uncertainty due to an errorful labeling across graphs, and we further explore the impact this lost information has on subsequent inference (see Section \ref{S:infoloss}).

In the presence of a latent vertex correspondence that is errorfully observed across graphs, graph matching methodologies can be applied to recover the latent vertex alignment before performing subsequent inference. 
Consequently, as multiple graph inference has surged in popularity, so has graph matching; see \cite{ConteReview} and \cite{foggia2014graph} for an excellent review of the graph matching literature.
Formally, given two graphs with respective adjacency matrices $A$ and $B$, the graph matching problem (GMP) seeks to minimize $\|A-PBP^T\|_F$ over permutation matrices $P$--- i.e., the GMP seeks a relabeling of the vertices of $B$ that minimizes the number of induced edge disagreements between $A$ and $PBP^T$; see Section \ref{sec:GM} for more detail. 
While the related graph isomorphism problem has recently been shown to be of sub-exponential complexity \cite{babai2016graph},
there are no efficient algorithms known for the more general problem of graph matching.
Due to its practical utility and computational difficulty, myriad heuristics have been proposed in the literature for approximately solving the GMP; see, for example, \cite{ConteReview} and \cite{FAQ} and the references contained therein.

Working in the aforementioned correlated stochastic blockmodel setting, 
we uncover a duality between graph matchability (see Definition \ref{defn:matchable}) and information loss.
We show that in the regime where graph matching can recover the latent vertex alignment after label shuffling, relatively little information is lost in the shuffle.  
We conjecture the inverse statement to be true as well:
In the regime where graph matching cannot recover the latent vertex alignment after shuffling, a relatively nontrivial amount of information is lost in the shuffle.
Formalizing graph matching as the antithetical operation to label shuffling allows us to better understand the utility of graph matching as a data preprocessing tool.  
Indeed, while in the presence of modest correlation relatively little information is lost due to shuffling, this lost information can have a dramatic negative effect on subsequent inference.
While this may seem like an indictment against joint inference in the errorful correspondence setting, we demonstrate that graph matching can effectively recover almost all of the lost information (see Theorem \ref{thm3}) and, consequently, the lost inferential performance.

\noindent{\bf{Note:}} Throughout, for real-valued function 
$f(\cdot):\mathbb{R}\mapsto\mathbb{R}$ and $g(\cdot):\mathbb{R}\mapsto\mathbb{R}$, we shall write $f(n)\sim g(n)$ if $\lim_{n\rightarrow \infty} f(n)/g(n)=1.$
We will also make use of the abbreviation a.a.s. (for asymptotically almost surely) which will be used as follows.  A sequence of events $E_n$ occurs a.a.s. if $\p(E_n^c)\rightarrow 0$ at a rate fast enough to ensure $\sum_n\p(E_n^c)<\infty$.

\section{Background and Definitions}
\label{sec:background}

We seek to understand the information lost due to the vertex correspondence across graphs being errorfully known, as well as the capacity of graph matching to recover this lost information.
In this section we provide a statistical framework and the necessary definitions amenable to pursuing these problems further.

\subsection{Correlated Stochastic Blockmodels}
\label{S:rhocorr}

The random graph framework in which we will anchor our analysis is the correlated stochastic blockmodel (SBM) random graph model of \cite{lyzinski_spectral}.
SBM's are widely used to model networks exhibiting an underlying community structure \cite{sbm,sbm2}, and provide a simple model family which has been effectively used to approximate the behavior of complex network data \cite{airoldi13:_stoch,wolfe13:_nonpar,choi2014co}.  Letting $\mathcal{G}_n$ denote the set of labeled, $n$-vertex, simple, undirected graphs, we define:
\begin{defn}
\label{def:rcsbm}
$(G_1,G_2)\in\mathcal{G}_n\times\mathcal{G}_n$ are $\rho$-correlated SBM($K,\vec{n},b,\Lambda$) random graphs (abbreviated $\rho$-SBM) 
if:

\vspace{2mm}

\noindent1.  $G_1=(V,E_1)$ and $G_2=(V,E_2)$ are marginally SBM($K,\vec{n},b,\Lambda$); i.e., for each $i=1,2$, 
\begin{itemize}
\item[i.] The vertex set $V$ is the union of $K$
 blocks $V_1$, $V_2$, \ldots, $V_K$, which are  
disjoint sets with respective cardinalities 
$n_1$, $n_2$, \ldots, $n_K$;
\item[ii.] The block membership function $b:V\mapsto[K]=\{1,2,\ldots,K\}$ is such that for each $v \in V$, $b(v)$ denotes the 
block of $v$; i.e., $v \in V_{b(v)}$;
\item[iii.] The block adjacency probabilities are given by the symmetric matrix $\Lambda\in [0,1]^{K\times K}$; i.e., for each 
 pair of vertices $\{j,\ell \} \in \binom{V}{2}$, the 
adjacency of $j$ and $\ell$ is an independent Bernoulli trial 
with probability of success $\Lambda_{b(j),b(\ell)}$.
\end{itemize}
2.  The random variables
$\left\{\mathds{1}[\{j,k\}\in E_i]\right\}_{i=1,2; \{j,k\}\in\binom{V}{2}}$
are collectively independent except that for each $\{j,k\}\in\binom{V}{2},$ the correlation between $\mathds{1}[\{j,k\}\in E_1]$ and $\mathds{1}[\{j,k\}\in E_2]$ is $\rho\geq0$.
\end{defn}

One of the keys to the theoretical tractability of the $\rho$-SBM model is that we can construct $\rho$-SBM($K,\vec{n},b,\Lambda$) random graphs $(G_1,G_2)$ as follows.  First draw $G_1$ from the underlying SBM($K,\vec{n},b,\Lambda$) model.  Conditioning on $G_1=(V,E_1)$, for each $\{j,\ell \} \in \binom{V}{2}$, if $\{j,\ell\}\in E_1$ then
$\mathds{1}[\{j,\ell\}\in E_2]$ is an independent 
Bernoulli trial with parameter 
$\Lambda_{b(j),b(\ell)}+\rho (1-\Lambda_{b(j),b(\ell)})$; if $\{j,\ell\}\notin E_1$ then
$\mathds{1}[\{j,\ell\}\in E_2]$ is an independent 
Bernoulli trial with parameter 
$\Lambda_{b(j),b(\ell)}(1-\rho)$. 
If $\rho\in(0,1)$ so that $G_1$ and $G_2$ are a.a.s.\@ \emph{not} isomorphic, this construction highlights a natural alignment between the vertex sets of $G_1$ and $G_2$: namely the identity function $\text{id}_n:[n]\mapsto[n]$. 
Indeed, for modest $\rho$ the identity function is (with high probability) the permutation of the vertex set of $G_2$ that best preserves the shared structure between $G_1$ and $G_2$;
see Theorem \ref{thm:no_match}. 
As, in practice, this alignment is often errorfully observed, we shall refer to $\text{id}_n$ as the {\it latent alignment} between $G_1$ and $G_2$.

\begin{remark}
\label{rem:ER}
\emph{Note that $\rho$-correlated Erd\H os-R\'enyi$(n,p)$, abbreviated $\rho$-ER$(n,p)$, random graphs (resp., $\rho$-correlated heterogeneous Erd\H os-R\'enyi$(P)$ random graphs) are easily realized by letting $K=1$ (resp., $K=n$) in Definition \ref{def:rcsbm}}
\end{remark}

\subsection{Shuffled $\rho$-correlated SBM random graphs}
\label{S:shuffle}

To understand the effect of an errorfully observed latent alignment function, we first need to define the action of errorfully aligning two $\rho$-SBM random graphs.  
Slightly abusing notation, we let $\Pi(n)$ denote both the set of permutation matrices and the set of permutations of $[n]$; to avoid confusion in the sequel, we will use the Greek letters $\phi,\tau,$ and $\sigma$ to denote permutations of $[n]$ and capital Roman letters $P$ and $Q$ to denote permutation matrices. 
For $x=(V,E_x)\in\sG_n,$ and $\phi\in \Pi(n),$ we define the {\it $\phi$-shuffled graph} $\phi(x)=(V,E_{\phi(x)})\in \sG_n$ via
$\{i,j\}\in E_x\text{ iff }\{\phi(i),\phi(j)\}\in E_{\phi(x)}$.
Equivalently, if the adjacency matrix of $x$ is $A_x$ and the permutation matrix associated with $\phi$ is $P_\phi$, then the adjacency matrix of $\phi(x)$ is $P_\phi A_x P_\phi^T.$

For a deterministic permutation $\phi$, the act of shuffling $(G_1,G_2)\sim\rho$-SBM is realized as follows. 
For all $x,y\in\mathcal{G}_n$,
$\p(G_1=x,\phi(G_2)=y):=\p(G_1=x,G_2=\phi^{-1}(y))$.
The action of randomly shuffling the vertices of $\rho$-SBM random graphs can then be defined via:
\begin{defn}
\label{def:shuffled}
Let $\boldsymbol{\sigma}$ be an $\Pi(n)$-valued random variable.
$(G_1,\boldsymbol{\sigma}(G_2))\in \mathcal{G}_n\times\mathcal{G}_n$ are $\boldsymbol{\sigma}$-shuffled, $\rho$-correlated SBM($K,\vec{n},b,\Lambda$) random graphs (abbreviated $\boldsymbol{\sigma},\rho$-SBM) if
\begin{itemize}
\item[i.] $(G_1,G_2)\sim\rho$-SBM($K,\vec{n},b,\Lambda$);
\item[ii.] For any $x,y\in\sG_n\times\sG_n$, we have 
\[\p(G_1=x,\boldsymbol{\sigma}(G_2)=y)=\sum_{\phi\in \Pi(n)}\p(\boldsymbol{\sigma}=\phi)\p\left(G_1=x,G_2=\phi^{-1}(y)\right).\]
\end{itemize}
Simply stated, we first realize $\boldsymbol{\sigma}$; conditioned on $\boldsymbol{\sigma}=\phi$, we then independently realize 
$(G_1,\phi(G_2))$.
\end{defn}

\noindent{\bf Note:}  In the sequel, we shall use $\phi$ and $\tau$ to denote  deterministic permutations, and $\boldsymbol{\sigma}$ to denote a permutation-valued random variable. 


\subsection{Graph matching and graph matchability}
\label{sec:GM}

If the latent alignment between $A$ and $B$ is errorfully known, graph matching methods can applied to approximately recover the true alignment. 
We formally define the graph matching problem as follows:
\begin{defn}
\noindent Given two graphs $n$-vertex graphs $G_1$ and $G_2$ with respective adjacency matrices $A$ and $B$, the graph matching problem (GMP) is defined as 
$\min_{P\in \Pi(n)}\|A-PBP^T\|_F.$
\end{defn}
\noindent Note that the GMP objective function $\|A-PBP^T\|_F$ is equal to $\|AP-PB\|_F$ and solving the GMP is equivalent to solving $\max_{P\in \Pi(n)}\text{trace}(APBP^T)$.  
Intuitively, solving the GMP is equivalent to relabeling the vertices of $G_2$ so as to minimize the number of induced edge disagreements between $G_1$ and $G_2$.

While solving the graph matching problem
is NP-hard in general, there are a bevy of approximation algorithms and heuristics in the literature that perform well in practice \cite{Zaslavskiy2009,FAQ,FAP,jovo,JMLR:v15:lyzinski14a} (in addition, see the excellent survey papers \cite{ConteReview,foggia2014graph} for a thorough review of the prescient literature and discussion of numerous alternate formulations of the GMP).
Note that in Section \ref{S:infoloss}, to approximately match the shuffled graphs in our synthetic and real data applications, we use FAQ algorithm of \cite{FAQ} and, when {\em seeded} vertices are present, the SGM algorithm of \cite{FAP}.
\emph{Seeded vertices}, or seeds, are those vertices whose latent alignments are known a priori and are not subjected to any label shuffling.

Note also that the graph matching problem is closely related to the problem of entity resolution/record linkage (see, for example, \cite{dedup1,dedup2}), especially in the setting of highly attributed networks.
In the present $\rho$-SBM setting, there is a key difference between the paradigms highlighted by the non-recoverability of vertex correspondences in the presence of general edge-shuffling; see Section \ref{s:nograph} for detail.

In the $\rho$-SBM setting the correlation structure across $G_1$ and $G_2$ highlights the natural alignment, namely $\text{id}_n$, between the two graphs.
In Theorems \ref{them:GMSBM} and \ref{thm:no_match}, we establish a phase transition for the values of $\rho$ under which graph matching can/cannot recover the latent alignment in the presence of vertex shuffling.  Before being able to state these results, we first must define the concepts of graph matchability and $\boldsymbol{\sigma}(G_2)$ matched to $G_1$.
\begin{defn}
\label{defn:matchable}
\noindent Let $(G_1,G_2)$ be vertex-aligned random graphs with respective adjacency matrices $A$ and $B$.  
We say that $G_1$ and $G_2$ are matchable if
$\text{argmin}_{P\in \Pi(n)}\|A-PBP^T\|_F=\{I_n\}.$
\end{defn}

To define the random graph $\boldsymbol{\sigma}(G_2)$ matched to $G_1$, we first define the concept of a matched graph for deterministic $x,y\in\mathcal{G}_n$. 
To this end, let
$P^*_{x,y}:=\text{argmin}_{P\in \Pi(n)}\|A_xP-PB_y\|_F.$
If $(G_1,G_2)=(x,y)$, then it is natural to define $\GM$, $G_2$ matched to $G_1$, as any element of
$P^*_{x,y}(y):=\left\{\phi(y)\text{ s.t. }P_{\phi}\in P^*_{x,y}\right\},$
with all elements of $P^*_{x,y}(y)$ being equally probable.  
Formally, we define
\begin{defn}
\label{def:GMM}
Let $(G_1,G_2)\sim\rho$-SBM($K,\vec{n},b,\Lambda$).
The $\mathcal{G}_n\times\mathcal{G}_n\times\mathcal{G}_n$-valued random variable $(G_1,G_2,\GM)$ has distribution defined via 
\begin{align*}
\p\big[(G_1,G_2,\GM)=(x,y,z)\big]
&=\p\big[(G_1,G_2)=(x,y)\big]
\frac{\mathds{1}\{z\in P^*_{x,y}(y)\}}{|P^*_{x,y}(y)|},\notag
\end{align*}
so that
$
\p[(G_1,\GM)=(x,z)]=\sum_{y}
\mathds{1}\{z\in P^*_{x,y}(y)\}\frac{\p[(G_1,G_2)=(x,y)]}{|P^*_{x,y}(y)|}.
$
\end{defn}
\noindent A consequence of Definition \ref{def:GMM} is that 
if $(G_1,\boldsymbol{\sigma}(G_2))\sim\boldsymbol{\sigma}$, $\rho-$SBM($K,\vec{n},b,\Lambda$),
then 
\begin{align*}
\p\big[(G_1,\GMs)=(x,z)\big]&=\p\big[(G_1,\GM)=(x,z)\big],\text{ and }\\
\p\big[(G_1,G_2,\GMs)=(x,y,z)\big]&=\p\big[(G_1,G_2,\GM)=(x,y,z)\big].
\end{align*}


\section{Information loss and graph matching}
\label{S:infolosstheory}

Given graph-valued random variables,
$(G_1,G_2)$, the mutual information of $G_1$ and $G_2$ is defined in the standard way via
$I(G_1,G_2)=\sum_{x,y\in\mathcal{G}_n}\p(G_1=x,G_2=y)\log\left(\frac{\p(G_1=x,G_2=y)}{\p(G_1=x)\p(G_2=y)}\right).$  Similarly,
 we define the entropy of $G_1$ via
$
H(G_1)=-\sum_{x\in\mathcal{G}_n}\p(G_1=x)\log(\p(G_1=x)).
$
If $\rho=0$, then two $\rho$-correlated SBM random graphs are independent, and the mutual information between them is $0$, regardless of whether the latent vertex alignment is known across graphs or not.  If $\rho=1$, then $G_1$ and $G_2$ are isomorphic and $I(G_1;G_2)=H(G_1)=H(G_2)$, the entropy of $G_1$.  If $\rho>0$, then there is nontrivial information shared across graphs, information which is potentially lost if the labeling is corrupted.  To this end, we have the following proposition, which is proved in Section \ref{sec:propI12}.

\begin{proposition}
\label{prop:I12}
Let $(G_1,G_2)\sim \rho$-SBM($K,\vec{n},b,\Lambda$). \\
\noindent i.  If $\Lambda$, $K$ and $\rho$ are fixed in $n$, then
$I(G_1;G_2)=\Theta(n^2)$.\\
\noindent ii. For fixed $\Lambda$ and $K$, if $\rho\rightarrow 0$ as $n\rightarrow \infty$ then 
$
I(G_1;G_2)\sim\rho^2\binom{n}{2}/2.
$\\
\noindent iii.  For fixed $\rho$ and $K$, if $p:=\max_{i,j}(\Lambda_{i,j})\rightarrow 0$ as $n\rightarrow \infty$ and $\min_i(n_i)=\Theta(n)$ then we have
$
I(G_1;G_2)\sim C p\rho\log\left(1+\frac{\rho(1-p)}{p}\right)n^2$, for a constant $C>0$.
\end{proposition} 
Proposition \ref{prop:I12} highlights the suitability of mutual information as a vehicle for studying graph correlation (and subsequently graph matchability).  
It is natural (in light of Theorems \ref{them:GMSBM} and \ref{thm:no_match}) to attempt to quantify the edge-wise correlation $\rho$ through the lens of graph matchability, as matchable graphs are precisely those whose correlation is above a phase transition threshold.
However, in the $\rho$-SBM setting the graph matching objective function computed at the latent alignment satisfies $\e(\|A-B\|_F^2)=cn^2(1-\rho)$ for a real constant $c>0$.
Likewise, the expected trace form of the graph matching objective function (shown in \cite{rel} to be preferable for capturing the true alignment operationally) computed at the latent alignment satisfies $\e(\text{trace}(AB))=n^2(c_1-c_2\rho)$ for real $c_1,c_2>0$.
In both cases, for correlation decaying to $0$ the lead order term is correlation independent, and neither readily captures the edge-wise dependency structure across graphs.
The mutual information, however, satisfies $I(G_1;G_2)\sim\rho^2\binom{n}{2}/2$.
The correlation in the lead order term emphasizes the utility of mutual information for teasing out graph correlation (and hence graph matchability) in the low correlation regimes.
Unfortunately, while computing $\|A-B\|_F^2$ and $\text{trace}(AB)$ is immediate, we are unaware of an efficient method for computing $I(G_1;G_2)$.
If available, computing a properly normalized version of $I(G_1;G_2)$ after matching would allow us to a posteriori judge the suitability of having matched the graphs in the first place.

\subsection{Information lost and matchability in the high correlation regime}
What is the degradation in information due to the uncertainty introduced by randomly permuting the labels of $G_2$ via $\boldsymbol{\sigma}$?
According to the information processing inequality, 
$ 
I(G_1;G_2)\geq I(G_1;\boldsymbol{\sigma}(G_2)) 
$
with equality if and only if $\boldsymbol{\sigma}$ has a point mass distribution.
Below we codify (see Theorems \ref{thm:infoloss} and \ref{them:GMSBM}) the following duality between graph matchability and information loss:
The correlation regime in which graph matching can successfully ``unshuffle'' the graphs---i.e., there is enough signal even in the shuffled graphs to recover the latent alignment---is precisely that in which relatively little information will be lost in the shuffle. 
Note that the proofs of Theorems \ref{thm:infoloss} and \ref{them:GMSBM} can be found in Section \ref{sec:pfinfoloss} and Section \ref{sec:proofSBMGM} respectively.

\begin{theorem}
\label{thm:infoloss}
Let $(G_1,G_2)\sim \rho$-SBM($K,\vec{n},b,\Lambda$), with $K$ and $\Lambda$ fixed in $n$, and let $\boldsymbol{\sigma}$ be uniformly distributed on $\Pi(n)$.
\begin{itemize}
\item[i.] For all values of $\rho$, it holds that
$I(G_1;G_2)-I(G_1;\boldsymbol{\sigma}(G_2))=O(n\log n).$
\item[ii.] 
If $\rho=\omega(\sqrt{\log n/n})$ and $\min_i n_i=\Theta(n)$, then
$I(G_1;G_2)-I(G_1;\boldsymbol{\sigma}(G_2))=\Omega(n\rho^2).$
\end{itemize}
\end{theorem}
\noindent 
If $\rho$ is constant in $n$, then the asymptotic upper bound of Theorem \ref{thm:infoloss} part i. and the asymptotic lower bound of Theorem \ref{thm:infoloss} part ii. differ only by a logarithmic factor.  
We suspect that the true order is $\omega(n\log n)$, as $H(\boldsymbol{\sigma})=\Theta(n\log n)$ is the loose upper bound we derive on the information loss in the proof of part i.\@ of the Theorem.
In addition, while Theorem \ref{thm:infoloss} is proven with $\boldsymbol{\sigma}$ uniformly distributed on $\Pi(n)$, we suspect that an analogous result holds for other distributions on $\Pi(n)$ that place suitable mass on permutations that shuffle $k = \Theta(n)$ elements of $[n]$, though we do not
pursue this further here.

\begin{remark}
\label{rem:withinblock}
\emph{In the proof of Theorem \ref{thm:infoloss} part ii., we essentially prove a stronger statement than that presented in the theorem. If we define 
$\Pi(n)^*:= \{\phi \in \Pi(n)\,\, |\,\, b(i) = b(\phi(i))\text{ for all }i \in [n]\},
$
to be the set of permutations that preserve vertex block
assignments and let $\boldsymbol{\sigma}^*$ be uniformly distributed on $\Pi(n)^*$, then we prove that under the assumptions of the theorem,
$I(G_1;G_2)-I(G_1;\boldsymbol{\sigma}^*(G_2))=\Omega(n\rho)$.
The information processing inequality (see Proposition \ref{prop:ipi2})
then gives us that $I(G_1;\boldsymbol{\sigma}(G_2))\leq I(G_1;\boldsymbol{\sigma}^*(G_2))$; indeed, there is information in the vertices’ block
assignments which is lost in $\boldsymbol{\sigma}$ and not in $\boldsymbol{\sigma}^*$.
Working with $\boldsymbol{\sigma}^*$ allows for errors of the vertex correspondences without having to deal with the mathematical complications that arise from also errorfully observed block memberships.
}
\end{remark}

In light of Proposition \ref{prop:I12}, relatively little information is lost due to shuffling in the $\rho=\omega(\sqrt{\log n/n})$ regime:  indeed, under this assumption on $\rho$ we have that 
$$\frac{I(G_1;G_2)-I(G_1;\boldsymbol{\sigma}(G_2))}{I(G_1;G_2)}=o(1).$$
Insomuch as graph matching is the antithetical operation to vertex shuffling, 
if relatively little information is lost in the shuffle then the graphs should be matchable (i.e., GM can unshuffle the networks); 
we formalize this below in Theorem \ref{them:GMSBM}.  
\begin{theorem}
\label{them:GMSBM}
With notation as above, let $A$ and $B$ be the adjacency matrices of $\rho-$SBM($K,\vec{n},b,\Lambda$) random graphs with $K$, and $\Lambda$ fixed in $n$. 
There exists a constant $\alpha>0$ such that if $\rho\geq \sqrt{\alpha\log n /n}$, then
$\p\big(\exists\, P\in\Pi(n)\setminus\{I_n\} \text{ s.t. }\|A-PBP^T\|_F<\|A-B\|_F  \big)=O\left(e^{- 3\log n}\right).$
\end{theorem}

\begin{remark}
\emph{
	We note here that results similar to Theorem \ref{them:GMSBM} for a much-simplified 2-block SBM appear in \cite{onaran2016optimal}, although the authors there consider a different MAP-based objective function in their matching setup.
}\end{remark}


\subsubsection{Shuffling sans graphs}
\label{s:nograph}

Considering an analogue of Theorem \ref{thm:infoloss} in the non-graph setting illuminates the special role that GM plays in recovering the lost information.
Consider the following example:  let $(X_1,Y_1),\ldots,$ $(X_n,Y_n)\stackrel{iid}{\sim}F_{XY}$, and suppose we observe the sequence of $X$'s and a shuffled version of the sequence of $Y$'s. 
What is the information loss due to this shuffling?
In the Bernoulli setting, this is partially answered by Theorem \ref{thm:infoloss}, as that theorem can be immediately cast in the classical setting with $(\mathds{1}_{\{u,v\}\in E_1},\mathds{1}_{\{u,v\}\in E_2})$ playing the role of $(X_i,Y_i)$.  
The key difference between the graph setting and the classical setting is the structure the graph imposes on the shuffling, as not all \emph{edge}-shuffles are feasibly obtained via vertex shuffles.
This structure is what allows GM to unshuffle the graphs, since optimizing the GM objective over all edge-shuffles (i.e., optimally unshuffling in the classical setting) would potentially induce significantly more edge-correlation than initially present in the graphs and would not be effective for recovering the lost vertex alignment.
Indeed, in this Bernoulli setting, it is easy to see that under mild model assumptions, we have $I(X,Y_{Y\mapsto X})>I(X,Y)$ where $Y_{Y\mapsto X}$ is the matched sequence of $Y$'s.


\subsection{The low correlation regime}
\label{sec:lowcorr}

In the $\rho$-SBM model, Theorem \ref{them:GMSBM} asserts that, under mild model assumptions, if $\rho$ is sufficiently large then $G_1$ and $G_2$ are matchable a.a.s.
In Theorem \ref{thm:no_match} below, we identify the second half of the \emph{matchability} phase transition at $\rho=\Theta(\sqrt{\log n/n})$.
Indeed, a consequence of Theorem \ref{thm:no_match} below is that, under mild assumptions, there exists a constant $\beta>0$ such that if $\rho\leq\sqrt{\beta \log n/n}$ then 
$G_1$ and $G_2$ are asymptotically \emph{not} matchable with probability $1$.
\begin{theorem}
\label{thm:no_match}
With notation as above, let $A$ and $B$ be the adjacency matrices of $\rho-$SBM($K,\vec{n},b,\Lambda$) random graphs with $K$, and $\Lambda$ fixed in $n$. 
Further assume there is an $\eta>0$ such that $\Lambda\in[\eta,1-\eta]^{K\times K}$.
Let $\{\tau_i\}_{i=1}^{N}$ be a collection of $N:=\sum_i\lfloor \frac{n_i}{2}\rfloor$ disjoint within-block transpositions; i.e., if $\tau_i=k\leftrightarrow\ell,$ then $b(k)=b(\ell)=b(\tau_i(k))=b(\tau_i(\ell))$. 
There exists a constant $\beta>0$ such that if
$\rho\leq\sqrt{\beta\log n/n},$ then 
$$\lim_{n\rightarrow\infty}\p\left( \bigcap_{i=1}^N \Big\{\|A-B\|_F<\|A-P_{\tau_i}BP_{\tau_i}^T\|_F\Big\}\right)= 0.$$
\end{theorem}

We conjecture a similar phase transition for the relative information loss due to $\boldsymbol{\sigma}.$
Namely, when $G_1$ and $G_2$ are not matchable we conjecture that a nontrivial fraction of the mutual information is lost in the shuffle.
\begin{conjecture}
\label{conj:lostinfo}
Let $(G_1,G_2)\sim \rho$-SBM($K,\vec{n},b,\Lambda$), with $K$ and $\Lambda$ fixed in $n$, and let $\boldsymbol{\sigma}$ be uniformly distributed on $\Pi(n)$.  If $\rho=o(\sqrt{\log n/n})$, then 
	$\frac{I(G_1;G_2)-I(G_1;\boldsymbol{\sigma}(G_2))}{I(G_1;G_2)}=\Theta(1).$
\end{conjecture}
\noindent While the matchability phase transition in Theorems \ref{them:GMSBM} and \ref{thm:no_match} is tighter than the conjectured phase transition in Conjecture \ref{conj:lostinfo}, if true, Conjecture \ref{conj:lostinfo} would imply a duality between information loss and matchability:  
\begin{align*}
1.&\text{ if }\rho=\omega(\sqrt{\log n/n})\text{ then }\frac{I(G_1;G_2)-I(G_1;\boldsymbol{\sigma}(G_2))}{I(G_1;G_2)}=o(1),\text{ and }G_1\text{ and }G_2\text{ are matchable}\\
2.&\text{ if }\rho=o(\sqrt{\log n/n})\text{ then }\frac{I(G_1;G_2)-I(G_1;\boldsymbol{\sigma}(G_2))}{I(G_1;G_2)}=\Theta(1),\text{ and }G_1\text{ and }G_2\text{ are not matchable}.
\end{align*}

In Section \ref{S:matching}, we show that in the high correlation regime---where the graphs are matchable and relatively little information is lost due to shuffling---graph matching can effectively recover the information lost due to shuffling a.a.s.; see Theorem \ref{thm3}.
This provides a theoretical foundation for understanding the utility of graph matching as a pre-processing step for a host of inference tasks:  often when the lost information due to shuffling has a negative effect on inference, graph matching can recover the lost information and improve the performance in subsequent inference.
In the $\rho=o(\sqrt{\log n/n})$ regime, while the graphs are no longer matchable, we conjecture (and experiments bear out) that the alignment found by graph matching still recovers much of the lost information.  
However, theoretically working in this regime will require new proof techniques, and we do not pursue this further here.


\subsection{Graph matching: Recovering the lost information}
\label{S:matching}

In Section \ref{S:infoloss} we show that the information lost, even when relatively small (see Theorem \ref{thm:infoloss}), can have a deleterious effect on subsequent inference, and
we demonstrate the potential of graph matching to recover the lost inference performance.
Theorem \ref{thm3} below provides a major step towards formalizing this intuition, proving that in the $\rho=\omega(\sqrt{\log n/n})$ regime, graph matching recovers almost all of the lost information.
The practical effect of this is the recovery of the lost inferential performance.
Note that the proof of Theorem \ref{thm3} can be found in Section \ref{sec:recover}.
\begin{theorem}\label{thm3}
Let $(G_1,\boldsymbol{\sigma}(G_2))\in \mathcal{G}_n\times\mathcal{G}_n\sim\boldsymbol{\sigma},\rho$-SBM($K,\vec{n},b,\Lambda$) with $\Lambda$ and $K$ fixed in $n$.  
If $\rho=\omega(\sqrt{\log n /n})$, then
$I(G_1;G_2)-I(G_1;\GMs)=o(1).$
\end{theorem}
\noindent 
Theorem \ref{thm3} provides a sharp contrast to the information lost in the shuffled graph regime in this high correlation setting.
Indeed, recall that Theorem \ref{thm:infoloss} implies that 
$I(G_1;G_2)-I(G_1;\boldsymbol{\sigma}(G_2))=\Omega(n\rho^2)=\Omega(\log n).$

\begin{remark}
\emph{The information processing inequality states that, given random variables $X$ and $Y$ and measurable $T$, 
$I(X;Y)\geq I(X;T(Y)).$
Intuitively, we cannot transform $Y$ independently of $X$ and increase the mutual information between $Y$ and $X$.
At first glance, Theorem \ref{thm3}, which implies that $I(G_1;\boldsymbol{\sigma}(G_2))<I(G_1;\GMs)$, seems to contradict this.
However, we note that $\GMs$ is a function of {\it both} $G_1$ and $\boldsymbol{\sigma}(G_2)$.
Indeed, if $Z=T(X,Y)$ then the information processing inequality need not hold (for a simple example, let $X$ have nontrivial entropy, let $Y$ be independent of $X$, and let $T(x,y)=x$).}\end{remark}


\section{Empirically matching in the low correlation regime}
\label{S:lowcorr}

\begin{figure*}[t!]
  \centering
  \subfloat[][Match error versus correlation]{\includegraphics[width=0.45\textwidth]{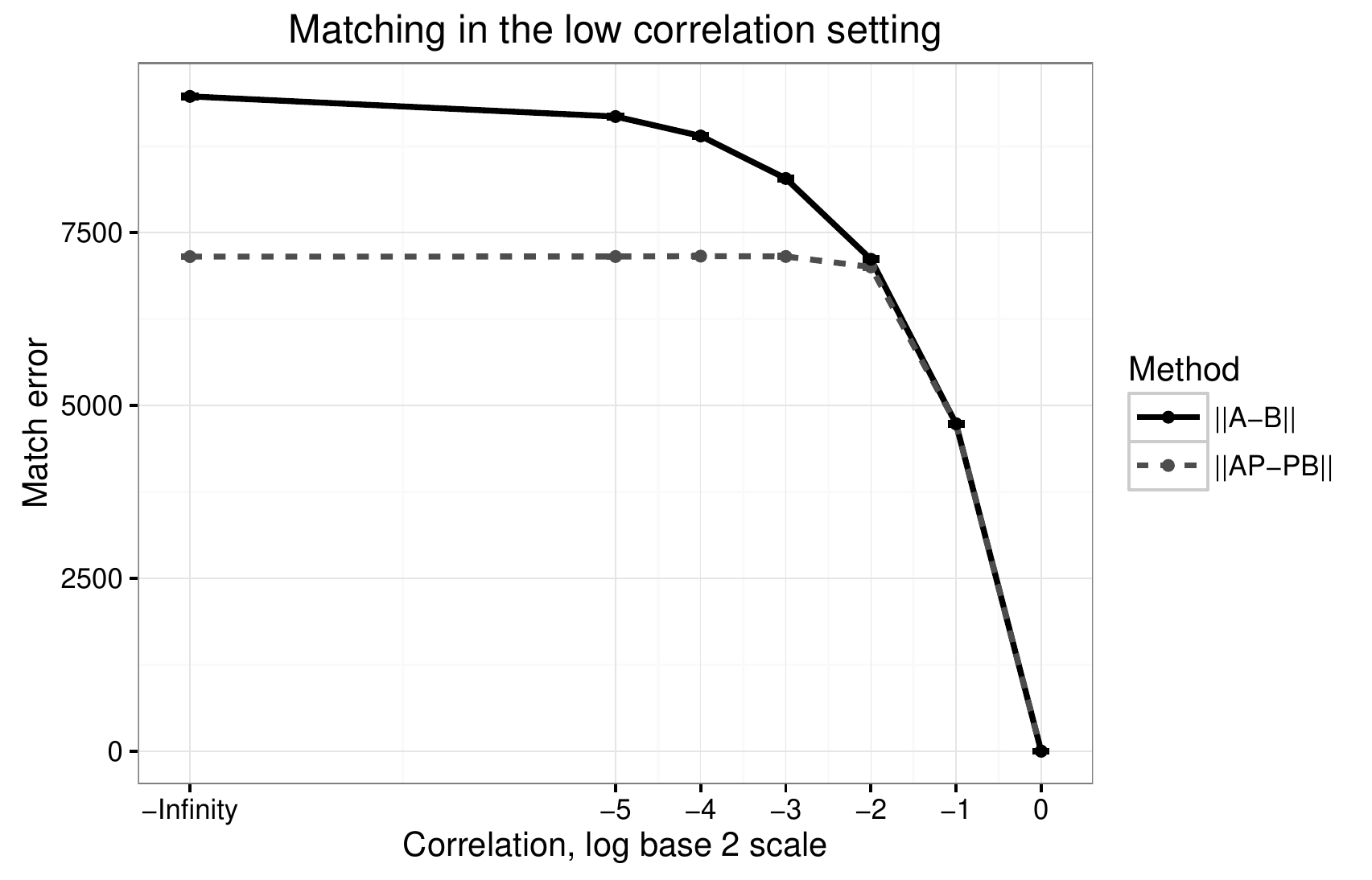}}
  \subfloat[][Sample edge correlation vs. correlation]{\includegraphics[width=0.45\textwidth]{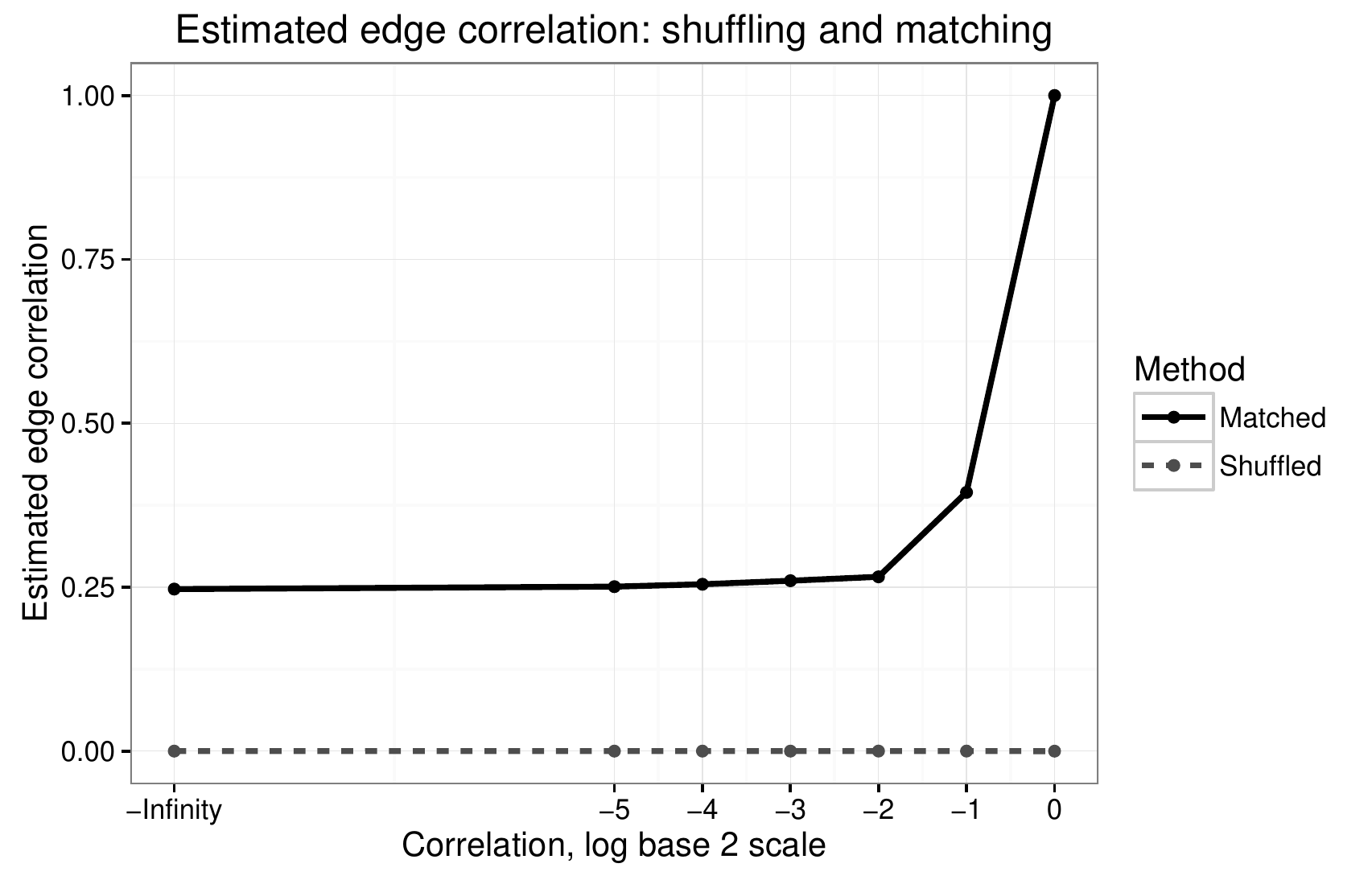}}
\caption{
For $(G_1,G_2)\sim\rho$-SBM$(3,(50,50,50),b,\Lambda)$,  
we plot the mean number of edge disagreements $\pm2$ s.e. against correlation $\rho$ when the graphs are aligned with the true latent correspondence (``$\|A-B\|$'' in the legend), and the graphs are aligned by initializing the FAQ graph matching algorithm at the true correspondence (``$\|AP-PB\|$'' in the legend).
In the right panel
we plot the mean sample edge correlation against $\rho$ for the matched graphs $(G_1,\GM)$ (``Matched'' in the legend), and for the shuffled graphs $(G_1,\boldsymbol{\sigma}(G_2))$ with $\boldsymbol{\sigma}$ uniformly distributed over $\Pi(150)$ (``Shuffled'' in the legend).
In each case, the means are computed across $200$ Monte Carlo iterates.
}
\label{fig:LC}
\end{figure*}

In the low correlation regime, graph matching with high probability cannot recover the true correspondence in the presence of vertex shuffling.
We conjecture that this is due, in part, to a nontrivial amount of the information between $G_1$ and $G_2$ being irrevocably lost due to the vertex shuffling.
To explore this further empirically, we consider the following experiment.
For $(G_1,G_2)\sim\rho$-SBM$(3,(50,50,50),b,\Lambda)$, with
$$b(i)=\mathds{1}_{[50]}(i)+2\mathds{1}_{51:100}(i)+3\mathds{1}_{101:150}(i)\text{ and }\Lambda=\begin{psmallmatrix} 0.5  &0.3& 0.2 \\ 0.3&0.5&0.3\\0.2&0.3&0.5\end{psmallmatrix},$$ 
we plot in the left panel of Figure \ref{fig:LC} the mean number of edge disagreements $\pm2$ s.e.\@ against correlation $\rho\in(0,1/32,1/16,1/8,$ $1/4,1/2,1)$ when the graphs are aligned with the true latent correspondence (``$\|A-B\|$'' in the legend), and the graphs are aligned by initializing the FAQ graph matching algorithm at the true correspondence (``$\|AP-PB\|$'' in the legend).
In each case, the means are computed across $200$ Monte Carlo iterates.
In the right panel
we plot the mean sample edge correlation against $\rho$ for the matched graphs $(G_1,\GM)$ (``Matched'' in the legend), and for the shuffled graphs $(G_1,\boldsymbol{\sigma}(G_2))$ with $\boldsymbol{\sigma}$ uniformly distributed over $\Pi(150)$ (``Shuffled'' in the legend), again averaged over $200$ Monte Carlo iterates. 
We note here that we observe similar phenomena as in Figure \ref{fig:LC} across a broad swath of parameter values as well.

From the figure we make the following observations.
First, as the FAQ algorithm is a Frank-Wolfe based approach, the agreement between the two methods in Figure \ref{fig:LC} panel (a) for $\rho\in(1/2,1)$ implies that for these large correlation levels, the latent alignment is a (local) optimum for the graph matching objective function.
Likewise, the improvement in the match error induced by matching the graphs via FAQ initialized at $I_{150}$ for $\rho\in(0,1/32,1/16,1/8)$ implies that the latent alignment is not a (global and local) optimum for the graph matching objective function for these lower correlation values.
This coincides with our intuition that in the presence (resp., absence) of enough correlation graph matching can (resp., cannot) recover the latent alignment in the presence of shuffling.
In the figure we also see that in the graph matched setting, the matched error roughly asymptotes after the matchability phase transition.  
This is due, in part, to the \emph{matchability} of $G_1$ and $\GM$---indeed, if the vertex alignment induced in $\GM$ is viewed as the true latent alignment then (absent symmetries) this alignment is clearly recoverable via graph matching.
Below the phase transition, matching the shuffled, correlated graphs artificially induces more edge-wise correlation than present in the latent alignment---see the right panel of Figure \ref{fig:LC}---in all cases bringing the edge correlation between $\GM$ and $G_1$ to the phase transition threshold.
Shuffling effectively makes the edges across graphs independent (uncorrelated equals independent in the Bernoulli setting), and 
matching the shuffled graphs induces the same local edge correlation structure as would matching independent graphs; the original edge-correlation structure which is captured in $I(G_1,G_2)$ is truly lost in the shuffle and the global structure preserved in the shuffle (subgraph counts, community structure, etc.) is not enough for graph matching to recover the latent alignment.


\section{The effect on subsequent inference}
\label{S:infoloss}

While the loss in information due to shuffling has little effect on inference tasks that are independent of vertex labels (for example, the nonparametric hypothesis testing methodologies of \cite{tang2014nonparametric,asta2014geometric}), the effect on inference that assumes an a priori known vertex alignment may be dramatic.
We demonstrate this in the context of joint graph clustering and two sample hypothesis testing for graphs.
The trend we demonstrate below is as follows: diminished performance as the alignment is shuffled and the performance loss due to shuffling being recovered via graph matching.
We expect this trend to generalize to a host of other joint inference tasks as well.
We note here that these results provide a striking contrast to those in \cite{vogelstein2011shuffled}, where it was shown that even if the graph labeling contain relevant class signal, obfuscating the labels does not necessarily decrease classification performance in a single graph setting.



\subsection{Hypothesis Testing}
\label{S:ERHT}
We first consider the simple setting of testing whether two $\rho$-correlated Erd\H os-R\'enyi graphs have the same edge probability.  Formally, given $(G_1,G_2)\sim\rho$-ER($n,p,q$)---i.e., the edgewise correlation of $G_1\sim$ER($n,p$) and $G_2\sim$ER($n,q$) is $\rho$ where, as $p$ and $q$ potentially differ here, we require $\rho\leq\min\left(\sqrt{\frac{p(1-q)}{q(1-p)}},\sqrt{\frac{p(1-q)}{q(1-p)}}\right)$---we wish to test the hypotheses
$H_0:p=q$ versus $H_1:p\neq q$.
If the edges across graphs were uncorrelated, under $H_0$ we can view the two edge sets as samples from independent Bin($\binom{n}{2},p$) random variables and a natural test statistic for testing $H_0$ versus $H_1$ would be that of the two proportion pooled $z$-test; namely
$T_1(G_1,G_2)=\frac{\hat p_1-\hat p_2}{\sqrt{2\hat p(1-\hat p)/\binom{n}{2}}}$
where $\hat p_1=|E_1|/\binom{n}{2}$, $\hat p_2=|E_2|/\binom{n}{2}$, and $\hat p=(\hat p_1+\hat p_2)/2$.
As the edges are positively correlated (with the same $\rho$) across graphs, a more powerful test of $H_0$ versus $H_1$ would be a paired two proportion $z$-test; namely that with test statistic
$T_2(G_1,G_2)=\frac{\hat p_1-\hat p_2}{\sqrt{\hat 2p(1-\hat p)(1-\hat\rho)/\binom{n}{2}}},$
where $\hat\rho$ is the empirical correlation between the edge sets of $G_1$ and $G_2$.
In this paired setting, directly applying the two proportion pooled test would yield an overly conservative (level less than $\alpha$) test.
Correcting for the type-I error in the pooled test (to make it approximately level $\alpha$) is achieved by multiplying the $z$-test critical value by $\sqrt{1-\hat\rho}$ yielding an equivalent test to the paired test using $T_2$.
A natural question in this paired setting is what level of shuffling is necessary for the pairedness of the data to become so corrupted (i.e., the information in the pairing being lost) as to render the unpaired test more powerful, and can graph matching recover the lost inferential performance in the paired regime?
Insomuch as $I(G_1;G_2)$ captures the edge-wise correlation across networks, this question is precisely addressing the effect that the reduced information (due to shuffling) has on testing power. 
\begin{figure}[t!]
  \centering
\includegraphics[width=0.7\textwidth]{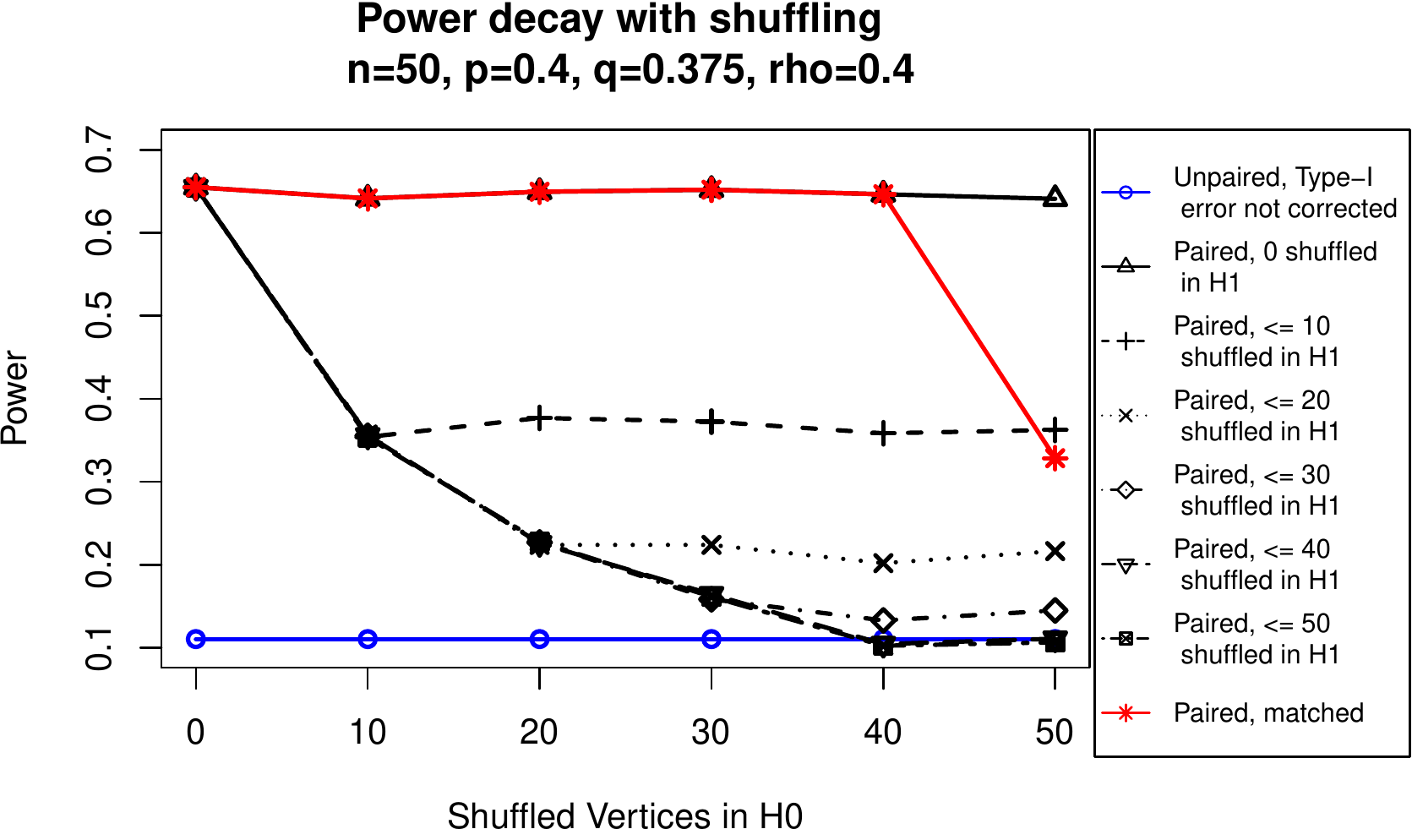}
\caption{
We plot the power (based on 2000 Monte Carlo trials) when testing $H_0:p=q$ versus $H_1:p\neq q$ against the number of unseeded (potentially shuffled) vertices, $n-s$.
In black, we plot the power of the paired test when at most $min(n-s,x)$ of the $n-s$ unseeded vertices have their labels shuffled under $H_1$ for $x\in(0,10,20,30,40,50)$.
The blue line plots the power of the unpaired test (not type-I error corrected), and the red line plots the power of the paired test when graph matching is used to align the networks before computing $T_2$.}
\label{fig:PowerLoss}
\end{figure} 

Exploring this further, we consider the above tests with $p=0.4$, $q=0.375$, $n=50$, and $\rho=0.7$ when $s$ vertex correspondences are assumed known across $G_1$ and $G_2$ for $s\in(0,10,20,30,40,50)$ (i.e., $n-s$ unseeded and potentially shuffled vertex labels), noting here that we observed similar phenomena across a broad swath of parameter choices.
As we can empirically sample from the null distribution of $T_2$ given the level of shuffling, to ensure a level $\alpha$ test the shuffled null distribution is computed under the least favorable element of the null hypothesis for $T_2$, which here corresponds to none of the $n-s$ unseeded vertex having their labels shuffled.

In Figure \ref{fig:PowerLoss}, we plot the power (based on 2000 Monte Carlo trials) when testing $H_0:p=q$ versus $H_1:p\neq q$ against the number of unseeded (potentially shuffled) vertices.
In black, we plot the power of the paired test when at most $min(n-s,x)$ of the $n-s$ unseeded vertices have their labels shuffled under $H_1$ for $x\in(0,10,20,30,40,50)$.
The blue line plots the power of the unpaired test (not type-I error corrected), and the red line plots the power of the paired test when graph matching is used to align the networks before computing $T_2$.
As expected, we see that as more vertices are shuffled under $H_1$ the power of the paired test decreases precipitously.
From the figure, we also see that the information which is lost in the the shuffle---and the subsequent lost testing power---is recovered by first graph matching the networks before computing $T_2$ (at least when $n-s<50$ and the graph matching is effective in recovering the latent correspondence).

Note that the matched test is more powerful than the unpaired alternative for all levels of shuffling in this example, and is more powerful than the shuffled test as long as the graphs are sufficiently shuffled under $H_1$.
As the exact level of shuffling amongst the unseeded vertices (in $H_0$ or $H_1$) is unknown a priori, we propose the graph matching version of the test as a conservative, more robust, version of testing $H_0$ versus $H_1$.
We view the decreased power at $n-s=50$ of the matched test as an algorithmic artifact; indeed, with no seeds the GM algorithm we employ often does not recover the true correspondence after shuffling.
With a perfect matching, we would expect the matched power to be $\approx0.67$ for all values of $n-s$.

An interesting aspect of Figure \ref{fig:PowerLoss} is that the power of the unpaired test and the paired test with $n-s=50$ are identical.  
For $T_1$ and $T_2$ to yield approximately the same power with the same critical value used (which is the case in the least favorable element of the null for $T_2$ with $0$ seeds), it is necessary that $\hat\rho\approx 0$.
As we are in a Bernoulli setting, the edge-wise correlation being effectively $0$ is indicative of the edges across graphs being effectively \emph{pairwise} independent (sample correlation has mean of order $1/n^2$), in which case the paired test we are using reduces to its unpaired alternative.
Although the edges are effectively pairwise independent, the graphs \emph{globally} are not.  
Indeed, less local structures (i.e., subgraph counts, community structures, etc.) are still correlated after shuffling.
In this high correlation setting, this global structure is also captured by $I(G_1,\boldsymbol{\sigma}(G_2))$ which, though diminished by shuffling, is still nontrivial.
It is this global structure that is able to be leveraged by graph matching to recover the lost local signal and the lost information.
This is precisely what separates the present graph setting from the more classical paired data setting, in which shuffling data labels is potentially irreversibly detrimental to subsequent inference.


\subsubsection{The effect of shuffling on embedding-based tests}
\label{S:Semipar}

In \cite{MT2}, we propose a semiparametric hypothesis testing framework for determining whether two graphs are generated from the same underlying random graph model. 
Note that while this test is not provably UMP (indeed, no such tests exist in the literature for sufficiently complex graph models), it is nonetheless one of the first provably consistent two-sample graph hypothesis tests posed in the literature.
The test proceeds as follows. Given $G_1$ distributed as heterogeneous ER($P$) and $G_2$ distributed as heterogeneous ER($Q$) with $P$ and $Q$ assumed positive semidefinite rank $d$ edge-probability matrices, $G_1$ and $G_2$ are first embedded into $\mathbb{R}^{n\times d}$ via adjacency spectral embedding.
       \begin{defn} 
  Let $G$ be an $n$-vertex graph with adjacency matrix $A$. 
  The $d$-dimensional adjacency spectral embedding (ASE$_d$) of $G_1$ is given by $\widehat X=U_{A}
 S_{A}^{1/2}$, where
$|A|=\big[U_{A}|\widetilde{U}_{A}\big]\big[{
  S}_{A} \oplus \widetilde{S}_{A}\big]\big[{
  U}_{A}|\widetilde{U}_{A}\big]$ is the spectral
decomposition of $|A| = (A^{T} A)^{1/2}$, 
$S_{A}\in\mathbb{R}^{d\times d}$ is the diagonal matrix containing the $d$ largest eigenvalues
of $|A|$ on its diagonal, and $U_{A}\in\mathbb{R}^{n\times d}$ is the matrix whose
columns are the corresponding orthonormal eigenvectors.
\end{defn}
\noindent In \cite{MT2} it is proven that, under mild assumptions, the (suitably rotated) rows of $\widehat X=ASE_d(G_1)$ concentrate tightly around the corresponding scaled eigenvectors of $P$ with high probability.
This fact is leveraged to produce a consistent hypothesis test for testing $H_0:P=Q$ versus $H_1:P\neq Q$ based on a suitably scaled version of the test statistic 
$T_1(\widehat X,\widehat Y)=    \min_{W \in \mathbb{R}^{d\times d}\text{s.t. } W^TW=I_d}\|\widehat X W-\widehat Y \|_{F},$ where $\widehat{X}=ASE_d(G_1)$ and $\widehat{Y}=ASE_d(G_2)$.

In \cite{levin2017central}, the test is further refined to more explicitly take advantage of the assumed known vertex correspondence across networks.
Inspired by \cite{jointLi} and the joint manifold embedding methodology of \cite{JOFC,fjofc} , we proceed as follows. 
Given vertex-aligned $G_1$ and $G_2$, we use ASE to first embed the \emph{omnibus adjacency matrix}
$\mathcal{O}=\begin{psmallmatrix}
       A & (A+B)/2\\
       (A+B)/2 &B\end{psmallmatrix}$.
By jointly embedding $G_1$ and $G_2$ via the omnibus matrix, the Procrustes rotation necessary in computing $T_1$ can be circumvented, which is empirically shown to increase testing power in \cite{levin2017central}.
To wit, if $ASE_d(\mathcal{O})=\begin{bsmallmatrix}\widehat X_O\\ \widehat Y_O\end{bsmallmatrix}$ with both $\widehat X_O,\widehat Y_O\in\mathbb{R}^{n\times d}$, then the omnibus test statistic is a suitably scaled version of $T_2(\widehat X_O,\widehat Y_O)=\|\widehat X_O-\widehat Y_O \|_{F}.$

As in the correlated Erd\H os-R\'enyi setting, we wish to understand the impact of vertex shuffling on testing power using $T_2$.
As before, we expect that as the labels are progressively more corrupted, tests based on graph invariants that do not utilize the latent alignment will achieve higher power than testing based on $T_2$.
Exploring this further, we consider the following experiment.  Let $P=XX^T$ be a rank $3$ positive semidefinite matrix with the rows of $X$ distributed as i.i.d. samples from a Dirichlet(1,1,1) distribution, and let $G_1$ be distributed as heterogeneous ER($P$) (so that $G_1$ can be viewed as a sample form a random dot product graph with parameter $X$ \cite{young2007random}).
Consider $Q=YY^T$ where the final $n-20$ rows of $Y$ are identical to those of $X$ and the first $20$ rows of $Y$ are realized via $Y(1\!:\!20,\cdot)=0.8*X(1\!:\!20,\cdot)+0.2*D$ where $D\in\mathbb{R}^{20\times 3}$ has i.i.d. Dirichlet(1,1,1) rows, independent of $X$.
We let $G_2$ be distributed as heterogeneous ER($Q$), with edges maximally correlated to those of $G_1$ by 
$\rho=\min\left(\sqrt{Q*(1-P)/(P*(1-Q))}, \sqrt{P*(1-Q)/(Q*(1-P))}\right),$
where the multiplication and division in the computation of $\rho$ is entry-wise for the matrices there involved.

\begin{figure}[t!]
  \centering
\includegraphics[width=0.7\textwidth]{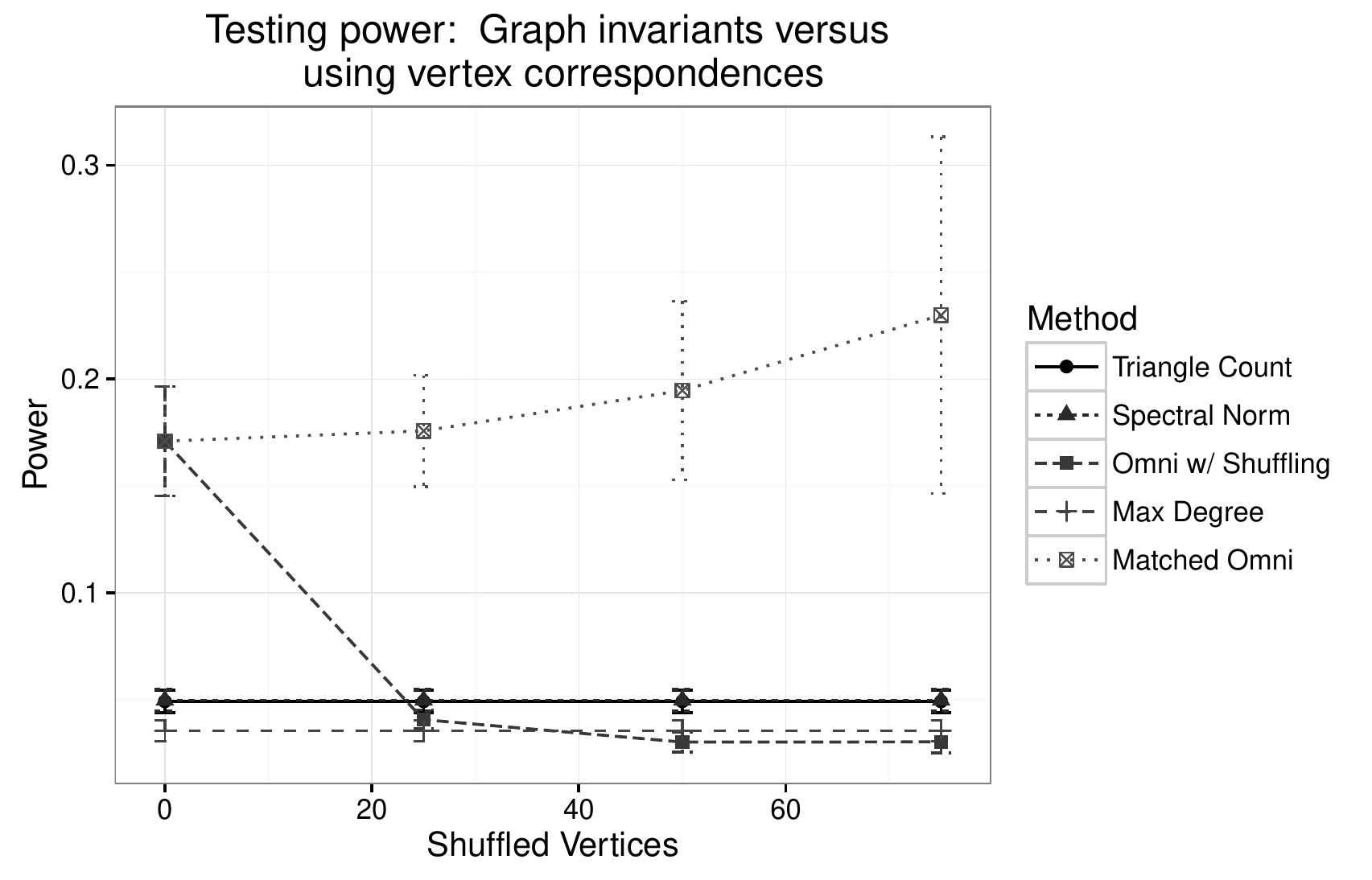}
\caption{
We plot (``Omni w/ Shuffling" in the legend) the power $\pm2$ s.e.\@ of the test of $H_0$ versus $H_1$ using $T_2$ (with appropriate Type-I error correction) when $x=(0,25,50,75)$ vertices have their labels potentially shuffled (i.e., $100-x$ seeded vertices).
We plot the power $\pm2$ s.e.\@ of the test using $T_2$ (``Matched Omni" in the legend) when the seeded vertices are used to first matched the graphs before computing $T_2$ to directly test $H_0$ versus $H_1$.
We plot the power $\pm2$ s.e.\@ of testing $H_0$ versus $H_1$ using test statistics $T_3=|M(A)-M(B)|$, $T_4=|t(A)-t(B)|$, $T_5=|s(A)-s(B)|$ (resp., ``Max Degree", ``Triangle Count", and ``Spectral Norm" in the legend).
In all cases, the power is averaged over 25 Monte Carlo replicates.}
\label{fig:versusGI}
\end{figure}

We are interested in understanding the power of the test using $T_2$ for detecting the anomalous behavior of the $20$ vertices in $G_2$. 
In Figure \ref{fig:versusGI}, we plot (``Omni w/ Shuffling" in the legend) the power $\pm2$ s.e.\@ of the test of $H_0$ versus $H_1$ using $T_2$ when $x=(0,25,50,75)$ vertices have their labels potentially shuffled (i.e., $100-x$ seeded vertices).
To control the level of the test using $T_2$ here, we sample the null distribution under the least favorable member of the composite null which corresponds to all unseeded vertices being shuffled under $H_0$.
Rather than considering different levels of shuffling under $H_1$ as in Section \ref{S:ERHT}, we plot the best possible performance for the composite alternative hypothesis; this is achieved when all unseeded vertices are also shuffled under $H_1$.
We also plot the power $\pm2$ s.e.\@ of the test using $T_2$ (``Matched Omni" in the legend) when the seeded vertices are used to first matched the graphs before computing $T_2$ to directly test $H_0$ versus $H_1$.
In each case, the power is averaged over 25 Monte Carlo replicates, which here corresponds to 25 different realizations of $G_1$ and $G_2$.

To compare the omnibus embedding based test to graph invariant based tests, we consider the following graph invariants.
For an adjacency matrix $A$, we let $M(A):=\max_i(\sum_j A_{i,j})$ be the maximum vertex degree in $A$, $t(A):=\text{trace}(A^3)/6$ be the number of triangle subgraphs present in $A$, and $s(A):=\|A\|_2$ be the spectral norm of $A$.  
In Figure \ref{fig:versusGI} we then plot the power $\pm2$ s.e.\@ of testing $H_0$ versus $H_1$ using these graph invariant test statistics: $T_3=|M(A)-M(B)|$, $T_4=|t(A)-t(B)|$, and $T_5=|s(A)-s(B)|$ (resp., ``Max Degree", ``Triangle Count", and ``Spectral Norm" in the legend).
In all cases, the power is averaged over 25 Monte Carlo replicates, which here corresponds to 25 different realizations of $G_1$ and $G_2$.

From the figure, we see that in the presence of sufficient shuffling the graph invariant based tests are more powerful than the test that leverages the errorful correspondence.  
However, graph matching successfully recovers the lost information and subsequently the lost testing power.
The increased variance in the graph matching based test at $x=75$ can be attributed to errors induced in the matching under $H_1$ with only 25 seeds.  
With this number of seeds our matching algorithm can effectively recover the latent alignment under $H_0$, while under $H_1$ the graph matching algorithm occasionally fail to recover the true correspondence.
These errors under $H_1$ result in strictly increased testing power here (and hence the increased standard error).  
While at first glance this would suggest using an imperfect matching to optimize power, the red curve demonstrates the dangers of having an imperfect matching in both $H_0$ and $H_1$.
As practically there is no way to know whether the matching is perfect or not or whether we are in $H_0$ or $H_1$, we view the matched test (always matching under both $H_0$ and $H_1$) as a conservative, robust alternative to its unmatched paired alternative.

We lastly note that while the graph invariant methods perform poorly in this heterogeneous ER anomaly setting, under alternate testing regimes we expect these tests to outperform the embedding based test here presented (for example, when testing in certain non-edge independent models).
Further understanding the properties of the underlying model that dictate this performance is paramount in practice, and we are presently pursuing this line of research.

\begin{remark}\emph{Note that while it is perhaps more natural to use an appropriately centered and scaled version of $\|A-B\|_F$ as our test statistic, 
as noted in \cite{MT2}, $\|A-B\|_F$ yields a test that is inconsistent for a large class of alternatives (e.g., if $G_2\sim ER(n,0.5)$ in the $G_1$ and $G_2$ independent setting), whereas the test based on $T_1$ is provably level-$\alpha$ consistent over the entire range of (fixed) alternative distributions.
In \cite{levin2017central}, we posit the same level-$\alpha$ consistency for testing based on $T_2$.}
\end{remark}

\subsection{Joint versus single graph clustering}
\label{Sec:clust}
We next explore the impact that label shuffling has on spectral graph clustering.  
Spectral graph clustering has become an important and widely-used machine learning method, with a sizable literature devoted to various spectral clustering algorithms under several model assumptions; see, for example, \cite{von2007tutorial,qin2013dcsbm,rohe2011spectral,sussman2012consistent,fishkind2013consistent,lyzinski2014perfect}.
We focus here on a variant of the methodology of \cite{sussman2012consistent,jointLi}, which embeds (a pair of) graphs into an appropriate Euclidean space and subsequently employs the $k$-means algorithm to cluster the data.  Here, rather than using $k$-means clustering to cluster the data, we will employ the model-based clustering algorithm \texttt{Mclust} \cite{fraley1999mclust}.

\begin{figure}[t!]
  \centering
  \includegraphics[width=0.5\textwidth]{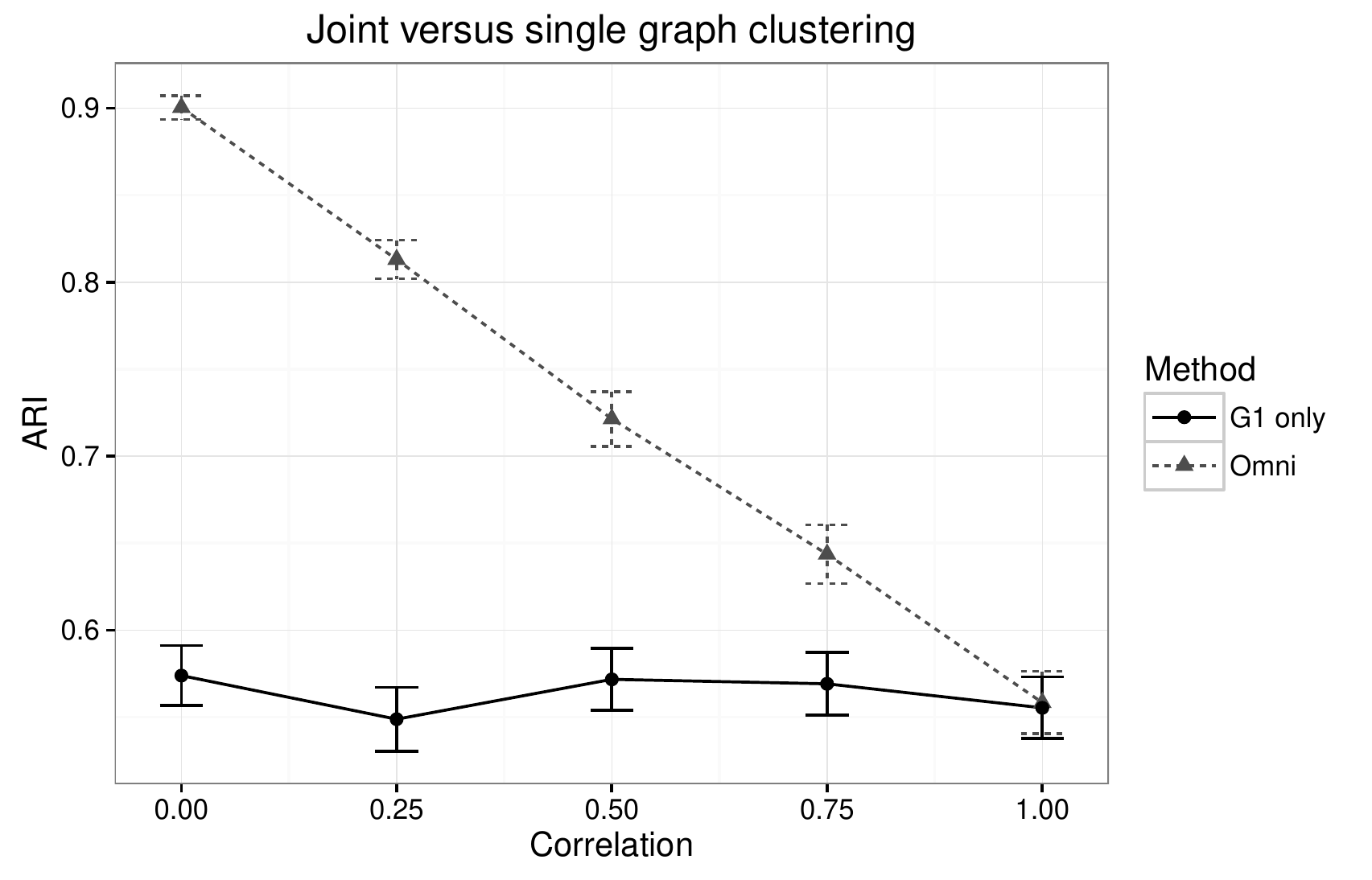}
  \caption{Joint versus single graph clustering of two $\rho$-correlated SBM.
Over a range of $\rho$, we embed $\mathcal{O}$ and cluster the jointly embedded vertices.
The dashed line plots the mean Adjusted Rand Index (ARI) $\pm$ 2 s.e. for the clustering of $G_1$ against its true block assignments when embedding $\mathcal{O}$ and jointly clustering the vertices.
The solid line plots the ARI $\pm$ 2 s.e. of clustering $ASE(G_1)$ against the true block assignments for single graph clustering.
In each case the number of Monte Carlo trials was 500.  }
  \label{fig:jvs}
\end{figure}
When we have multiple graph valued observations of the same data, can we efficiently utilize the information between the graphs to increase clustering performance?
In the manifold matching literature, there are numerous examples of this heuristic:
leveraging the signal across multiple data sets can increases inference performance within each of the data sets (see, for example, \cite{JOFC,fjofc,sun2013generalized,shen2014manifold}).
Inspired by this, given vertex-aligned $G_1$ and $G_2$, we use ASE to embed the Omnibus adjacency matrix $\mathcal{O}$
and use \texttt{Mclust} to cluster the embedded vertices.

\begin{figure*}[t!]
  \centering
  \subfloat[][$\rho=0.5$]{\includegraphics[width=0.45\textwidth]{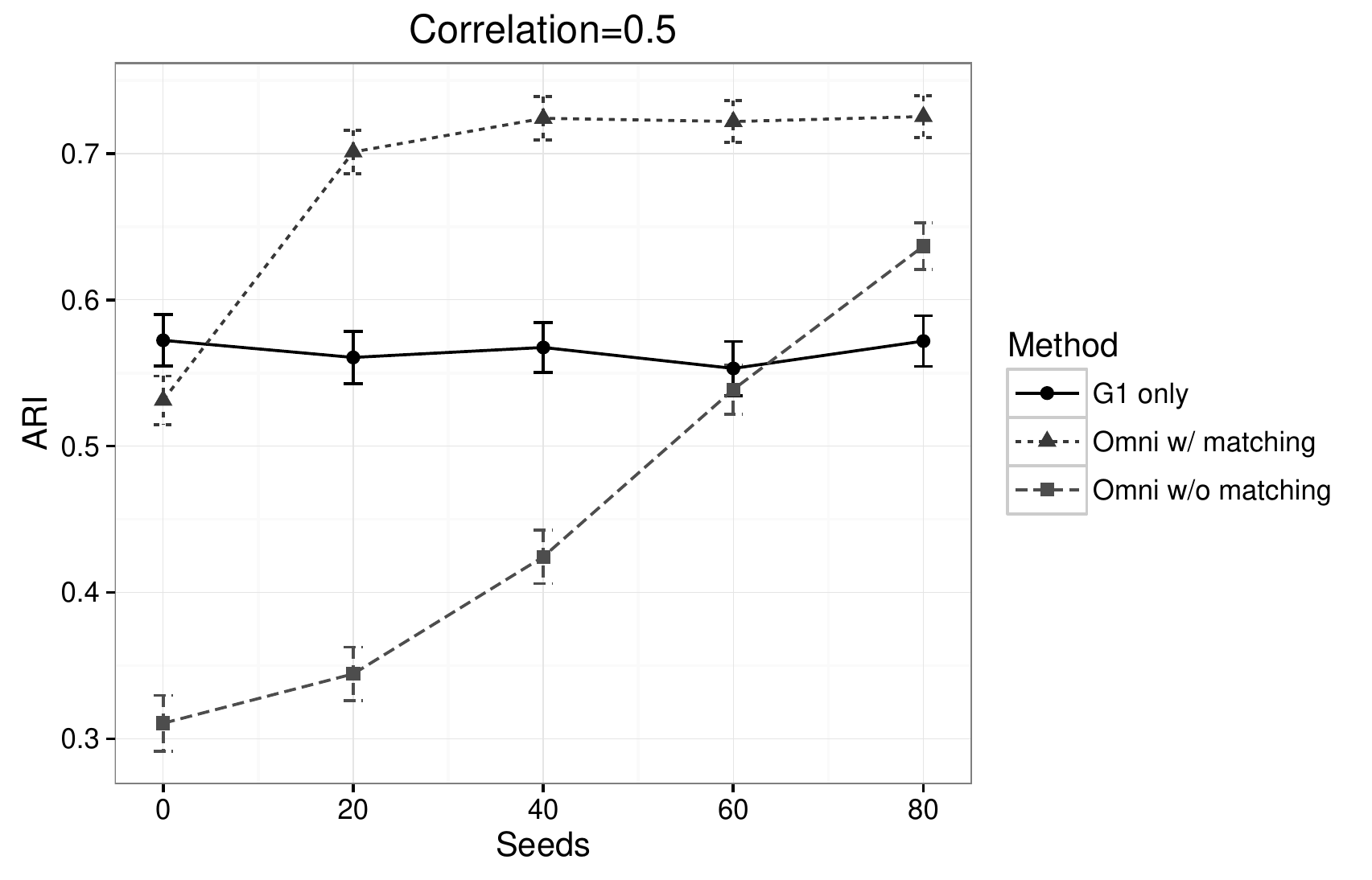}}
  \subfloat[][$\rho=0.7$]{\includegraphics[width=0.45\textwidth]{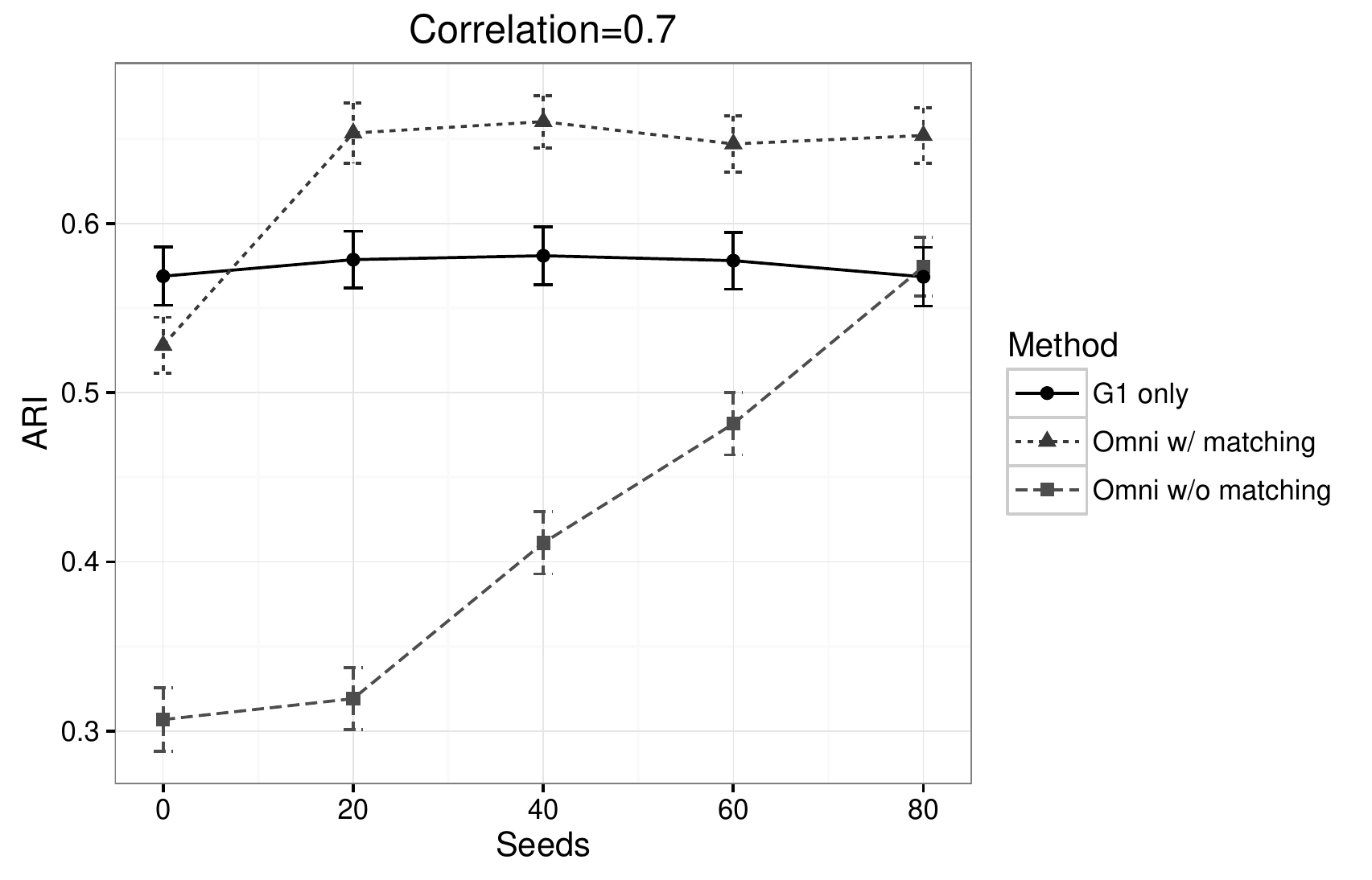}}
\caption{
We plot the mean ARI $\pm 2$ s.e.\@ of: i) the clustering of $G_1$ obtained via jointly embedding/clustering when $100-s$ of the vertices in $G_2$ have their labels randomly permuted (labeled ``Omni w/o matching'' in the legend); ii) embedding and clustering applied to $G_1$ only (labeled ``G1 only'' in the legend); iii) the clustering of $G_1$ obtained via jointly embedding/clustering after matching the shuffled $B$ back to $A$ (labeled ``Omni w/ matching'' in the legend).  In the left plot, the across graph correlation is $0.5$, and is $0.7$ in the rightmost plot.  }
\label{fig:gmvs}
\end{figure*} 


To demonstrate the potential performance increase achievable via jointly embedding $G_1$ and $G_2$ versus a separate embedding, 
we consider 
$(G_1,G_2)\sim\rho-\text{SBM}\left(2, \vec{n}=(50,50), b, \binom{0.1\,\,\,\,0.05}{0.05\,\,\,\,0.2}\right)$.
      Results are displayed in Figure \ref{fig:jvs}.
Over a range of $\rho$, we jointly embed and cluster $G_1$ and $G_2$ via $ASE(\mathcal{O})$ and \texttt{Mclust} with the dashed line plotting the mean Adjusted Rand Index \cite{rand1971objective} (ARI) $\pm$ 2 s.e. for the clustering obtained for $G_1$ against its true block assignments.
We also embed and cluster $G_1$ alone, and the solid line plots the mean ARI $\pm$ 2 s.e. of the obtained clustering of $G_1$ against its true block assignments.
In each case the number of Monte Carlo trials was 500.  
Across these synthetic experiments, we used the true $d=k=2$. 
In Figure \ref{fig:jvs}, we see 
significantly improved clustering accuracy achieved by joint inference
for modest to lowly correlated $(G_1,G_2)$.
Note that as the correlation increases, the increased performance due to the borrowed strength of joint inference diminishes.  
This is unsurprising as the amount of additional information added by $G_2$ is less for larger $\rho$ (indeed if $\rho=0$ then $H(G_2|G_1)=H(G_2)=O(n^2)$ while if $\rho=1$ then $H(G_2|G_1)=0$). 

Does this increased performance due to joint inference degrade in the presence of an errorfully observed vertex correspondence?
To explore this further, we randomly permute the labels of $100-s$ vertices in $B$, so that there are $s\in\{0,20,40,60,80\}$ seeded vertices whose labels are kept true (note that 
these seeded vertices are randomly chosen from the 100 total vertices). 
We plot the performance of joint clustering pre- and post-graph matching for $\rho=0.5$ and $\rho=0.7$ in Figure \ref{fig:gmvs}.
In light of Theorem \ref{thm:infoloss}, we see that the modest information lost due to the shuffling dramatically decreases the performance of our paired graph inference.
This can be readily explained: the information lost due to shuffling is precisely the information leveraged by the joint embedding, namely the vertex correspondences across graphs.

\begin{figure*}[t!]
  \centering
  \subfloat[][Chemical connectome]{\includegraphics[width=0.5\textwidth]{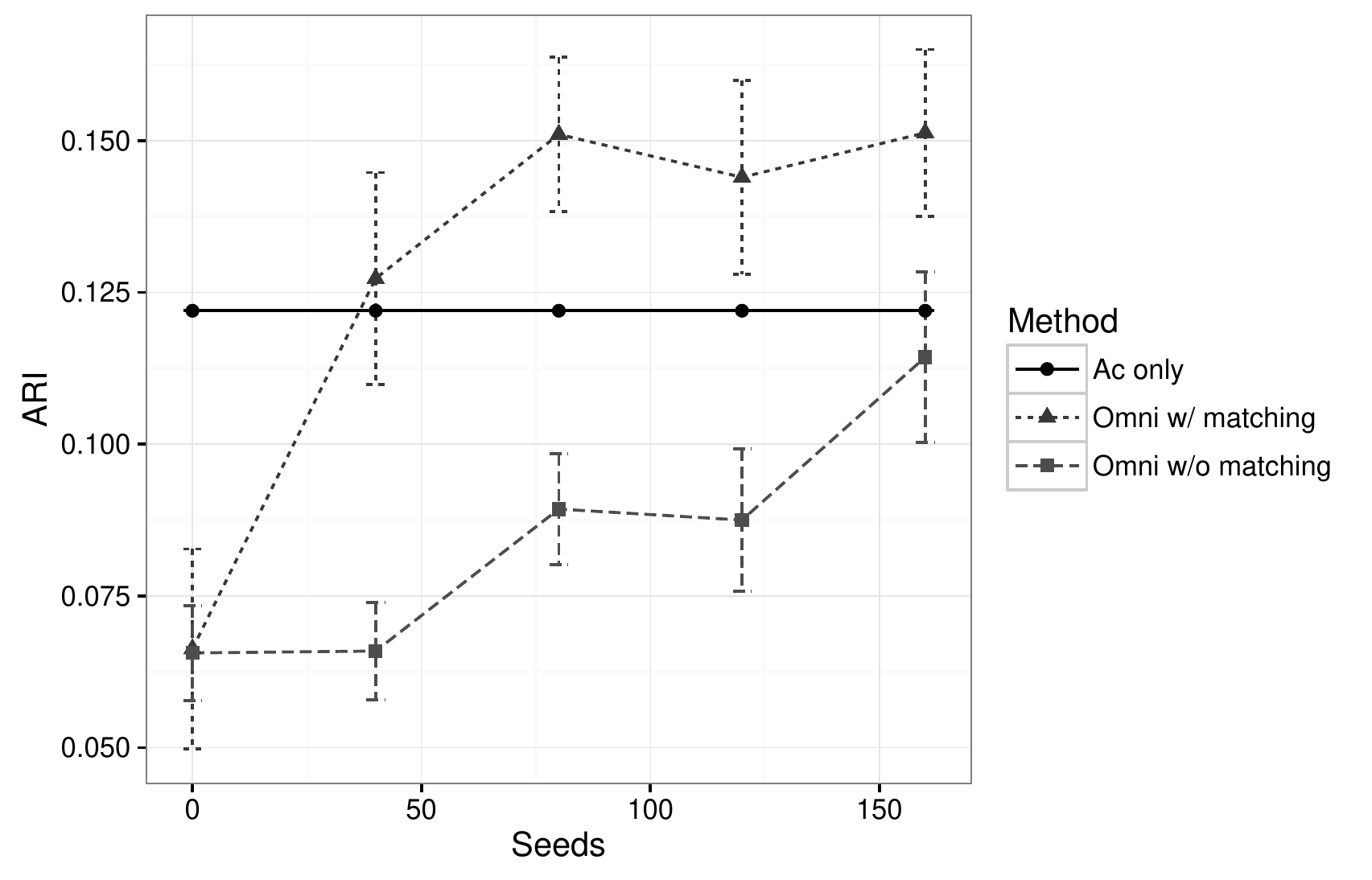}}
  \subfloat[][Electrical connectome]{\includegraphics[width=0.5\textwidth]{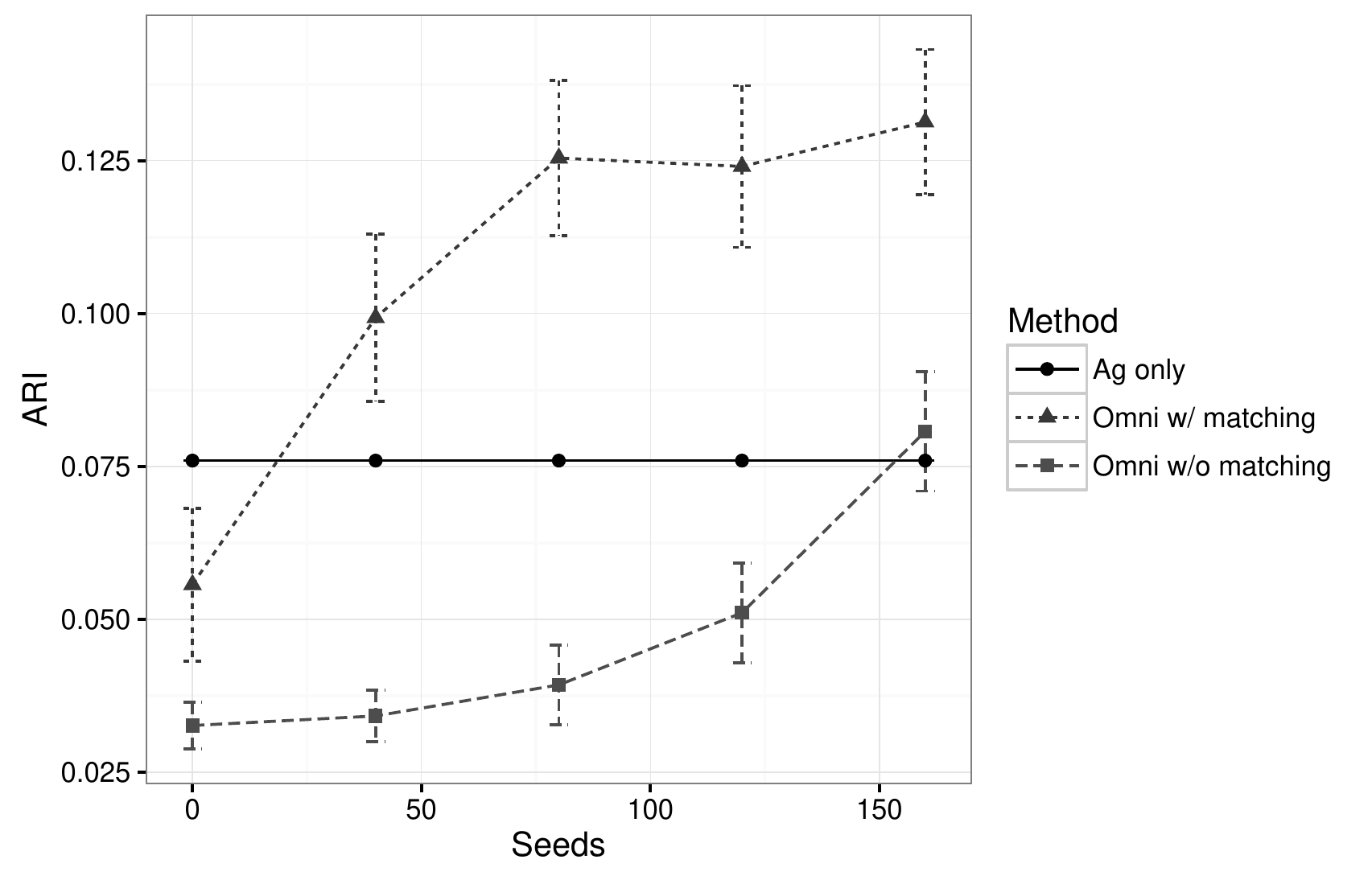}}
\caption{We plot the mean ARI $\pm 2$ s.e.\@ (against the true clustering of the connectome into inter, motor, and sensory neurons) of: i) jointly embedding/clustering when $247-s$ of the vertices in $G_2$ (chosen uniformly at random) have their labels randomly permuted (labeled ``Omni w/o matching'' in the legend); ii) embedding and clustering only a the single graph (labeled ``Ac only'' and ``Ag only'' in the legend); iii) jointly embedding/clustering after matching the shuffled $B$ back to $A$ (labeled ``Omni w/ matching'' in the legend).  In the left plot, we plot the performance on the chemical connectome, and on the right the electrical connectome.
Each plot is computed with 50 Monte Carlo trials.}
\label{fig:worms}
\end{figure*}
In Figure \ref{fig:gmvs}, we also explore the inferential impact of graph matching in recovering this lost performance.  
In light of Theorem \ref{thm3}, we see that graph matching recovers the information lost in shuffling, and therefore recovers much of lost performance.
It is notable that when $s=0$, GM using FAQ performs little better than chance in recovering the true across graph labeling (recall, exact GM is NP-hard and note that seeding was shown in \cite{FAP} to dramatically increase GM performance).  
We note that FAQ does (errorfully) recover the across-graph block assignments, and so aligns the graphs in a way that preserves some of the necessary structure leveraged by the joint embedding/clustering procedure. 
This suggests an extension of Theorems \ref{thm3}, in which perfect matching is not needed for a significant portion of the information lost by shuffling to be fully recovered. 

Note also that in the lower correlation setting, the performance after matching is significantly better than the corresponding performance in the higher correlation setting (relative to the base single graph inference level).  
This confirms the intuition proposed by Figure \ref{fig:jvs}:  With higher correlation, there is less additional information to be recovered by matching (although fewer seeds are needed to recover this lost information), and therefore, there is a comparatively smaller increase in performance achieved by the joint inference even post matching.

As a second example, we consider jointly clustering the {\it C. elegans} connectomes \cite{varshney2011structural}.
The connectome of the {\it C. elegans} roundworm has been completely mapped and neurons interact with each other in two distinct modes: via electrical gap junctions and via chemical synapses.
In \cite{lyzinski2014seeded,jointLi}, the authors showed that the electrical and chemical connectomes contain complimentary signal, and both papers suggest that inference should proceed in the joint graph space.  
To explore this further, as in \cite{jointLi}, we pre-process the data by removing the isolated neurons (under either modality) and symmetrizing each connectome.
The resulting connectomes each have 253 vertices which are classified into 3 neural types:  motor neurons, sensory neurons, and inter neurons.
In the left (resp., right) panel of Figure \ref{fig:worms}, we plot the performance of jointly embedding/clustering the vertices in the chemical (resp., electrical) connectome versus embedding and clustering the single graphs.
Throughout, we estimate $d$ via computing an elbow in the SCREE plot and use \texttt{Mclust} to cluster the data into $k=3$ clusters.

In both modalities, significantly better clustering performance is achieved by working in the joint graph space; indeed, in the case where the correspondence is perfectly observed across graphs, the ARI for the joint embedding/clustering of the chemical (resp., electrical) connectome is $0.15$ (resp., $0.18$) compared to $0.12$ (resp., $0.08$) ARI via single graph regime.  
The larger increase in the performance of jointly clustering the electrical connectome versus the chemical connectome can be attributed to the stronger cluster structure present in the chemical connectome.
Indeed, the electrical connectome can borrow the strength of that signal when the graphs are jointly embedded, whereas there is not much auxiliary signal present in the electrical connectome for the chemical connectome to leverage.

We lastly note that these graphs are particularly challenging to match, with only $\approx 10\%$ of the vertices correctly matched by the state-of-the-art SGM algorithm even with 150 seeds (see \cite[Figure 4]{FAP}).
Nonetheless, the structure that is uncovered by graph matching (namely the recovery of vertex classes across graphs) is enough to recover a large portion of the performance increase seen in the perfectly matched omnibus embedding setting.
Again, this suggests an extension of Theorems \ref{thm3}, in which perfect matching is not needed for a significant portion of the information lost by shuffling to be fully recovered.


\section{Discussion and future work}
Many joint graph inference procedures assume that the vertex correspondence between graphs is known a priori.  
However, in practice the correspondence may only be partially known or errorfully known, and we seek to understand the effect that errors in the labeling have on subsequent inference. 
To this end, we provide an information theoretical foundation for answering the following questions: What is the increase in uncertainty (i.e., loss of the mutual information) between two graphs when the labeling across graphs is errorfully observed, and can this lost information be recovered via graph matching?
Working in the correlated stochastic blockmodel setting, we prove that when graph matching can perfectly recover an errorfully observed correspondence (Theorem \ref{them:GMSBM}), relatively little information is lost due to shuffling (Theorem \ref{thm:infoloss}).
However, we demonstrate that this lost information can have a dramatic effect on the performance of subsequent inference.  
We also show that asymptotically almost all of the lost information can be recovered via graph matching (Theorem \ref{thm3}), which has the effect of recovering much of the lost inferential performance.

In the process, we are able to establish a phase transition for graph matchability at $\rho=\Theta(\sqrt{\log n/n})$ for $(G_1,G_2)\sim\rho$-SBM.  
We prove in Theorems \ref{them:GMSBM} and \ref{thm:no_match} that under mild assumptions there exists constants $0<\beta<\alpha$ such that $\rho\geq\sqrt{\alpha\log n/n}$ implies $G_1$ and $G_2$ are matchable and $\rho\leq\sqrt{\beta\log n/n}$ implies $G_1$ and $G_2$ are unmatchable.
We conjecture the analogous phase transition at $\rho=\Theta(\sqrt{\log n/n})$ for the relative information loss due to shuffling in Conjecture \ref{conj:lostinfo}.
Establishing Conjecture \ref{conj:lostinfo} would cement a 
duality between the information lost due to shuffling and matchability (i.e., the ability to undo the shuffling via graph matching): if $\rho=\omega(\sqrt{\log n/n})$ then $\frac{I(G_1;G_2)-I(G_1;\boldsymbol{\sigma}(G_2))}{I(G_1;G_2)}=o(1)$, and $G_1$ and $G_2$ are matchable; 
if $\rho=o(\sqrt{\log n/n})$ then $\frac{I(G_1;G_2)-I(G_1;\boldsymbol{\sigma}(G_2))}{I(G_1;G_2)}=\Theta(1)$, and $G_1$ and $G_2$ are not matchable.
The difficulty in proving the conjecture lies in lower bounding the mutual information in the mixture model $(G_1,\boldsymbol{\sigma}(G_2))$ with low correlation.

While graph matching cannot correctly recover the lost correspondence in the low correlation setting, $\rho=o(\sqrt{\log n/n})$, we suspect---and the experiments of Section \ref{S:infoloss}, demonstrate---that a significant portion of the lost information is recovered even by an imperfect matching.
This would have at least two immediate consequences. 
First, the estimated correlation across different real data networks---even on the same vertex sets---is often very small.
A theorem proving that GM can recover much of the lost information with an imperfect matching in these low correlation regimes would further highlight the applicability of GM across a broad class of data sets.
Second, as graph matching is NP-hard in general and no algorithm exists that can perfectly match even modestly sized graphs, extending Theorem \ref{thm3} to the case of an imperfect matching would serve to further highlight the practical utility of graph matching {\it algorithms}.

Moreover, across correlation levels we see the dramatic algorithmic impact of seeds on our ability to successfully unshuffle networks.
It is natural to ask what the theoretical impact of seeding is in terms of lost information.
For seeds chosen uniformly at random, the impact on information loss is marginal:  If $cn$ of the vertices are seeded for constant $c\in(0,1)$, then the bounds of Theorem \ref{thm:infoloss} are affected only by a constant multiplicative factor and not in their order of magnitude.
This is a stark contrast to the algorithmic impact obtained by seeds, as in \cite{JMLR:v15:lyzinski14a} it is shown that only logarithmically many random seeds are needed to asymptotically perfectly match graphs.
While this suggests that randomly chosen seeds have a more practical impact on algorithmic performance than in mitigating information loss, we suspect that well-chosen seeds will have a measurable impact on the bounds in Theorem \ref{thm:infoloss}.
While this is outside the scope of the present manuscript, we are actively pursuing seed selection strategies both from the theoretical and practical perspectives.

Lastly, can we extend this theory to a broader class of random graph models? 
Note that while the results of Proposition \ref{prop:I12} translates \emph{mutatis mutandis} to the more general $\rho$-correlated heterogeneous Erd\H os-R\'enyi model considered in \cite{rel}, the methods for proving Theorems \ref{them:GMSBM} and \ref{thm:no_match} do not, hence the present focus on the $\rho$-SBM model.
Also note that deriving the analogous results to Theorems \ref{them:GMSBM}--\ref{thm3} for directed, weighted graphs with a block weighting structure (i.e., the edge-weights depend only on the block membership function $b$ of the in and out vertices) is immediate given almost surely nonnegative and bounded weight distributions.  
While stochastic blockmodels (weighted and unweighted) are widely used to model data with latent community structure, they are an overly simplistic model for many real data applications.
We are working to extend the theory and results to non-edge-independent random graph models (for example, power law graphs, bounded degree graphs, etc.).
However, these non-edge-independent graphs require 
a novel correlation structure and new graph matching theory to be developed.

\section{Acknowledgments}
We wish to thank Donniell Fishkind, Carey Priebe, Daniel Sussman, Minh Tang, Avanti Athreya and Joshua T. Vogelstein for their comments and suggestions which greatly helped the exposition and ideas in this paper.

\appendix
\section{Proofs and supporting results}
\label{sec:proofs}

Herein, we collect the proofs of the main theorems and supporting results. 
Note that for $\phi\in \Pi(n)$ we define 
$s(\phi):=\{i\in[n]\text{ s.t. }\phi(i)\neq i\}$
to be the number of vertices shuffled by $\phi$.
To ease notation, for $g_1,g_2\in\mathcal{G}_n$ we will adopt the following shorthand: 
\begin{align*}
\p(g_1,g_2):=\p(G_1=g_1,G_2=g_2),&\hspace{5mm}\p(g_1):=\p(G_1=g_1)=\p(G_2=g1)\\
\ps(g_1,g_2):=\p(G_1=g_1,\boldsymbol{\sigma}(G_2)=g_2),&\hspace{5mm}\ps(g_2):=\p(\boldsymbol{\sigma}(G_2)=g_2).
\end{align*}


\subsection{Proof of Proposition \ref{prop:I12}}
\label{sec:propI12}

In this section, we will provide a proof of Proposition \ref{prop:I12}.  
Recall that the random variables
$\left\{\mathds{1}[\{j,k\}\in E(G_i)]\right\}_{i=1,2; \{j,k\}\in\binom{V}{2}}$
are collectively independent except that for each $\{j,k\}\in\binom{V}{2},$ the correlation between $\mathds{1}[\{j,k\}\in E(G_1)]$ and $\mathds{1}[\{j,k\}\in E(G_2)]$ is $\rho\geq0$.  Next note that if $(X,Y)\sim\rho$-correlated Bern$(p)$ then
\begin{align}
\label{eq:expandthis}
I(X;Y)&=p(p+\rho(1-p))\log\left(1+\rho\frac{(1-p)}{p} \right)+2p(1-p)(1-\rho)\log(1-\rho)\notag\\
&\hspace{25mm}+(1-p)(1-p+p\rho)\log\left(1+\rho\frac{p}{1-p}\right).
\end{align}
Together this yields
\begin{align}
\label{eq:infoexpand}
I(G_1;G_2)&=\sum_{g_1,g_2\in\mathcal{G}_n}\p(g_1,g_2)\log\left(\frac{\p(g_1,g_2)}{\p(g_1)\p(g_2)} \right)\notag\\
&=\sum_{\substack{i,j\in[K]\\i< j}}^Kn_{i,j}\bigg[\Lambda_{i,j}(\Lambda_{i,j}\!+\!\rho(1-\Lambda_{i,j}))\log\!\left(\!1\!+\rho\frac{1-\Lambda_{i,j}}{\Lambda_{i,j}} \right)\!+\!2\Lambda_{i,j}(1-\Lambda_{i,j})(1-\rho)\log(1-\rho)\notag\\
&\hspace{15mm}+(1-\Lambda_{i,j})(1-\Lambda_{i,j}+\Lambda_{i,j}\rho)\log\left(1+\rho\frac{\Lambda_{i,j}}{1-\Lambda_{i,j}}\right)\bigg],
\end{align}
where 
$n_{i,j}=n_in_j$ if $i\neq j$  and $n_{i,j}=\binom{n_i}{2}\text{ if }i=j$.
Expanding Eq. (\ref{eq:expandthis}) for $p$ fixed and $\rho\rightarrow0$ and applying this to Eq. (\ref{eq:infoexpand}) yields part ii. of the proposition, and expanding Eq. (\ref{eq:expandthis}) for $\rho$ fixed and $p\rightarrow0$ and applying this to Eq. (\ref{eq:infoexpand}) yields part iii. of the proposition.



\subsection{Proof of Theorem \ref{thm:infoloss}}
\label{sec:pfinfoloss}

If $\boldsymbol{\sigma}$ is uniformly distributed on $\Pi(n)$, then for any fixed $\tau\in \Pi(n)$, it is immediate that $
\ps(\tau(g_2))
=\ps(g_2).
$
We also have that
\begin{align*}
I(G_1;\boldsymbol{\sigma}(G_2))&=\sum_{g_1,g_2\in\mathcal{G}_n}\ps(g_1,g_2)\log\left(\frac{\ps(g_1,g_2)}{\p(g_1)\ps(g_2)} \right)\\
&=\sum_{g_1,g_2\in\mathcal{G}_n}\left(\sum_{\phi\in \Pi(n)}\frac{1}{| \Pi(n)|}\p(g_1,\phi(g_2))\right)\log\left(\sum_{\tau\in \Pi(n)}\frac{1}{| \Pi(n)|}\frac{\p(g_1,\tau(g_2))}{\p(g_1)\ps(g_2)}\right)\\
&=\sum_{\phi\in \Pi(n)}\frac{1}{| \Pi(n)|}\sum_{g_1,g_2\in\mathcal{G}_n}\p(g_1,g_2)\log\left(\sum_{\tau\in \Pi(n)}\frac{1}{| \Pi(n)|}\frac{\p(g_1,\tau\circ\phi^{-1}(g_2))}{\p(g_1)\ps(\phi^{-1}(g_2))}\right).
\end{align*}
We then have that $I(G_1;G_2)-I(G_1;\boldsymbol{\sigma}(G_2))$ is equal to
\begin{align}
\label{eq:infodiff}
&-\sum_{\phi_\in{\Pi(n)}}\frac{1}{|\Pi(n)|}\sum_{g_1, g_2\in \mathcal{G}_n}\p(g_1,g_2)\log\left(\sum_{\tau\in \Pi(n)}\frac{\p(g_1,\tau\circ\phi^{-1}(g_2))}{|\Pi(n)|\cdot\p(g_1,g_2)} \right)+\sum_{g_2\in \mathcal{G}_n}\p(g_2)
\log \left(
\frac{\ps(g_2)}
{\p(g_2)}
\right).
\end{align}
\begin{proof}[Proof of Theorem \ref{thm:infoloss} part i.]
We begin by proving part $i.$ of the theorem.  Note that 
\begin{align}
\label{eq:relent}
\sum_{g_2\in \mathcal{G}_n}\p(g_2)
\log \left(
\frac{\ps(g_2)}
{\p(g_2)}
\right)
&=
-H(\boldsymbol{\sigma}(G_2))+H(G_2)\leq0.
\end{align}
where the final inequality follows from the concavity of the entropy function $H(\cdot)$.
Applying this to Eq. (\ref{eq:infodiff}) then yields
\begin{align*}
I(G_1;G_2)-I(G_1;\boldsymbol{\sigma}(G_2))\leq&-\sum_{\phi_\in{\Pi(n)}}\frac{1}{|\Pi(n)|}\sum_{g_1, g_2\in \mathcal{G}_n}\p(g_1,g_2)\log\left(\sum_{\tau\in \Pi(n)}\frac{\p(g_1,\tau\circ\phi^{-1}(g_2))}{|\Pi(n)|\p(g_1,g_2)} \right)\notag \\
=&-\sum_{\phi_\in{\Pi(n)}}\frac{1}{|\Pi(n)|}\sum_{\substack{g_1\in \mathcal{G}_n,\\ g_2\in \mathcal{G}_n}}\p(g_1,g_2)\log\left(\frac{1}{|\Pi(n)|}+\sum_{\substack{\tau\in \Pi(n)\\  \tau\neq \phi}}\frac{\p(g_1,\tau\circ\phi^{-1}(g_2))}{|\Pi(n)|\p(g_1,g_2)}\right)\\
\leq&-\sum_{\phi_\in{\Pi(n)}}\frac{1}{|\Pi(n)|}\sum_{\substack{g_1\in \mathcal{G}_n,\\ g_2\in \mathcal{G}_n}}\p(g_1,g_2)\log\left(\frac{1}{|\Pi(n)|}\right)=\log(|\Pi(n)|)\sim(n\log n),
\end{align*}
as desired.
\end{proof}
\noindent\emph{Proof of Theorem \ref{thm:infoloss} part ii.}
Let $(G_1,G_2)\sim\rho$-SBM($K,\vec n,b,\Lambda$).
To prove part ii. of Theorem \ref{thm:infoloss}, we will consider permutations in 
$$\Pi(n)^*:=\{\phi\in \Pi(n)\,|\,b(i)=\phi(b(i))\text{ for all }i\in[n]\},$$
i.e., permutations of the vertex sets of $G_1$ and $G_2$ that fix block assignments.
For $\boldsymbol{\sigma}^*$ uniformly distributed in $\Pi(n)^*$, we will show that 
$$I(G_1;G_2)-I(G_1;\boldsymbol{\sigma}^*(G_2))=\Omega(n\rho^2).$$
To complete the proof, we use the information processing inequality to show that $I(G_1;\boldsymbol{\sigma}(G_2))\leq I(G_1;\boldsymbol{\sigma}^*(G_2))$ (see Proposition \ref{prop:ipi2} for detail).

We now establish some notation and preliminary results.  
For $\phi\in \Pi(n)^*$, write
$\phi=(\phi_1,\phi_2,\ldots,\phi_K),$ where $\phi_i:V_i\mapsto V_i$ is the restriction of $\phi$ to $V_i$.
\begin{defn}
Let $(G_1,G_2)$ be $\rho$-correlated SBM($K,\vec{n},b,\Lambda$) random graphs, and let $x,y\in \mathcal{G}_n$. 
\begin{itemize}
\item[1.]For each of $i=1,2$, and $j=1,2,\ldots,K$, let $G_i^j=G_i\!\!\restriction_{V_j}$ (resp., $x^j=x\!\!\restriction_{V_j}$, $y^j=y\!\!\restriction_{V_j}$) be the induced subgraph of $G_i$ (resp., fixed $x,y\in\mathcal{G}_n$) restricted to $V_j$. 
\item[2.]
For $j,\ell\in[K]$ with $j<
\ell$, let $G_i^{j,\ell}$ (resp., $x^{j,\ell}$, $y^{j,\ell}$) be the induced bipartite subgraph of $G_i$ (resp., $x$, $y$) composed of the edges in $G_i$ (resp., $x$, $y$) between vertex sets $V_j$ and $V_\ell$.  
\end{itemize}
\end{defn}
\noindent The key to restricting our attention to permutations in $\Pi(n)^*$ is the following.  Working with permutations in $\Pi(n)^*$ allows us to split the SBM random graphs along block assignments, and then tackle each block (and each bipartite between-block) subgraph separately.  
Before formalizing this in Claim \ref{claim:key}, we will need to define the analogues of correlated bipartite graphs.  
To this end, if $\mathcal{B}_{m_1,m_2}$ is the set of all labeled bipartite graphs $G=(U,V,E)$ with $|U|=m_1,$ $|V|=m_2,$ and $E\subset U\times V,$ then we define:
\begin{defn}
Two $m_1m_2$-vertex bipartite random graphs $(G_1,G_2)\in\mathcal{B}_{m_1,m_2}\times \mathcal{B}_{m_1,m_2}$ are $\rho$-correlated Bipartite($m_1,m_2,p$) random graphs (abbreviated $\rho$-Bipartite) if 
\begin{itemize}
\item[i.] For each $i=1,2,$ $G_i\in\mathcal{B}_{m_1,m_2}$, and edges between the bipartite sets $U$ and $V$ are independently present with common probability $p$;  
\item[ii.] The random variables
$\left\{\mathds{1}[(j,k)\in E(G_i)]\right\}_{i=1,2; j\in U,k\in V}$
are collectively independent except that for each $(j,k)\in U\times V,$ the correlation between $\mathds{1}[(j,k)\in E(G_1)]$ and $\mathds{1}[(j,k)\in E(G_2)]$ is $\rho\geq 0$.
\end{itemize}
\end{defn}
\noindent The deterministic shuffling of $\rho$-Bipartite graphs can be defined completely analogously to the $\phi$-shuffled graph of Section \ref{S:shuffle}.  To wit, if $x\in\mathcal{B}_{m_1,m_2},$ and $\tau\in S_{m_1}$, $\phi\in S_{m_2},$ we define the {\it $[\tau,\phi]$-shuffled graph}, denoted by $[\tau,\phi](x)=(V,E_{[\tau,\phi](x)})\in \mathcal{B}_{m_1,m_2}$, via
$(i,j)\in E_x\text{ iff }(\tau(i),\phi(j))\in E_{[\tau,\phi](x)};$
i.e., $U$ is shuffled according to $\tau$ and $V$ is shuffled according to $\phi.$
The following claim is immediate.
\begin{claim}
\label{claim:key}
Let $(G_1,G_2)\sim\rho$-SBM($K,\vec n,b,\Lambda$), and let $\phi\in \Pi(n)^*$.
\begin{itemize} 
\item[1.] $\p(G_2=y)=\p(G_2=\phi(y))$ for all $y\in\sG_n$;
\item[2.] For each $j=1,2,\ldots,K$, $(G_1^j,G_2^j)$ are distributed as $\rho$-ER($n_j,B_{j,j}$) random graphs;
\item[3.] For $j,\ell\in[K]$ with $j<\ell$, $(G_1^{j,\ell},G_2^{j,\ell})$ are distributed as $\rho$-Bipartite$(n_j,n_{\ell},B_{j,\ell})$;  
\item[4.] The collection of graph pairs
$$\big\{(G_1^i,G_2^i)\big\}_{i\in[K]}\bigcup\big\{(G_1^{j,\ell},G_2^{j,\ell})\big\}_{j,\ell\in[K],\,j< \ell}$$ is mutually independent.
\end{itemize}
\end{claim}

Analogues of Theorem \ref{them:GMSBM} hold in the $\rho$-ER and $\rho$-Bipartite settings as well. 
The following Lemma is proved similarly to Theorem \ref{them:GMSBM}, and so the proof is only briefly sketched.
\begin{lemma} \label{lem:preludeto10}
With notation as above,\\
\noindent i) If $(G_1,G_2)\sim\rho$-ER($m,p$) with respective adjacency matrices $A$ and $B$. If $\rho=\omega\left(\sqrt{\frac{\log m}{m}}\right)$ then there exists a $C>0$ such that 
\begin{align}
\label{eq:restrictedER}
\p(\exists\,\phi\in \Pi(m)\setminus\{\text{id}_m\}\text{ s.t. }\|A-P_\phi BP_\phi^T\|_F^2-\|A-B\|_F^2\leq C m\rho)=O(e^{-3\log m}).
\end{align}
\noindent ii)
If
$(G_1,G_2)\sim\rho$-Bipartite($m_1,m_2,p$)
with respective adjacency matrices $A$ and $B$, and let $m=\min(m_1,m_2)$.  
If $\rho=\omega\left(\sqrt{\frac{\log m}{m}}\right)$ then there exists a $C>0$ such that 
\begin{align}
\label{eq:restrictedBip}
\p(\exists\,&(\phi,\tau)\in \Pi(m_1,m_2)\setminus\{\text{id}_{m_1,m_2}\}\text{ s.t. }\|A-P_\phi BP_\tau^T\|_F^2-\|A-B\|_F^2\leq C m\rho)=O(e^{-3\log m}),
\end{align}
where $\Pi(m_1,m_2)=\Pi(m_1)\times \Pi(m_2)$, and $\text{id}_{m_1,m_2}=(\text{id}_{m_1},\text{id}_{m_2}).$
\end{lemma}
\begin{proof}
We will sketch the proof of part $i)$ with part $ii)$ following mutadis mutandis.
For the moment, fix $\phi\in \Pi(m)$ with $s(\phi)=k.$
For $x,y\in\mathcal{G}_n$, define 
\begin{align}
F_{\mathcal{A}}(x,y,\phi)&:=\left\{\{u,v\}\in\binom{V}{2}\text{ s.t. }u\nsim_x v, \, u\sim_{\phi(x)} v, \text{ and }u\nsim_{\phi(y)} v\right\};\label{eq:Fa}\\
F_{\mathcal{O}}(x,y,\phi)&:=\left\{\{u,v\}\in\binom{V}{2}\text{ s.t. }u\sim_x v, \, u\nsim_{\phi(x)} v, \text{ and }u\sim_{\phi(y)} v\right\}.\label{eq:Fo}
\end{align}
We then have (if $x$ and $y$ have adjacency matrices $A$ and $B$)
\begin{align*}
\frac{1}{2}(\|A- P_\phi B P_\phi^T\|_F^2-\|A- B\|_F^2)&=\frac{1}{2}\|A-P_\phi AP_\phi^T\|_F^2-2F_{\mathcal{A}}(x,y,\phi)-2F_{\mathcal{O}}(x,y,\phi).
\end{align*}
With $\mathfrak{m}_k$ defined via $\mathfrak{m}_k=\binom{k}{2}+k(m-k)$, applying \cite[Proposition 3.2]{kim} to $(G_1,G_2)\sim\rho$-ER($m,p$) yields that there exists a constant $c$ (which can be taken to be $\sqrt{48p(1-p)}$) such that for $m$ sufficiently large
\begin{align}
\label{eq:APAPTedgesconcentrate}
\p&\left(G=g\text{ s.t. }\left|\frac{1}{2}\|A_g-P_\phi A_gP_\phi^T\|_F^2-2\mathfrak{m}_k p(1-p)\right|\geq 
2c\sqrt{\mathfrak{m}_k k\log m}\right)\leq  2e^{-3k\log m}. 
\end{align}
Conditioning on $\frac{1}{2}\|A-P_\phi AP_\phi^T\|_F^2=\Delta$, 
$F_{\mathcal{O}}(G_1,G_2,\phi)\sim Bin(\Delta/2,p(1-\rho))$ independent of $F_{\mathcal{A}}(G_1,G_2,\phi)\sim Bin(\Delta/2,(1-p)(1-\rho))$.
Hoeffding's inequality then yields 
\begin{align}
\label{bound2}
\p\left[F_{\mathcal{O}}(G_1,G_2,\phi)\geq \frac{\Delta}{2} \left(p(1-\rho)+\frac{\rho}{3}\right)\bigg|\frac{1}{2}\|A-P_\phi AP_\phi^T\|_F^2=\Delta\right]&\leq e^{-\Delta\rho^2/9};\\
\label{bound3}
\p\left[F_{\mathcal{A}}(G_1,G_2,\phi)\geq \frac{\Delta}{2} \left((1-p)(1-\rho)+\frac{\rho}{3}\right)\bigg|\frac{1}{2}\|A-P_\phi AP_\phi^T\|_F^2=\Delta\right]&\leq e^{-\Delta\rho^2/9}.
\end{align}
Unconditioning (\ref{bound2})--(\ref{bound3}) combined with (\ref{eq:APAPTedgesconcentrate}), and summing over $k$ yields the desired result.
\end{proof}

Next, note that 
\begin{align*}
I(G_1;G_2)-I(G_1;\boldsymbol{\sigma}^*(G_2))&=
-\sum_{\substack{g_1\in \mathcal{G}_n,\\ g_2\in \mathcal{G}_n}}\p(g_1,g_2)\log\left(\sum_{\tau\in \Pi(n)^*}\frac{\p(g_1,\tau(g_2))}{\p(g_1,g_2)| \Pi(n)^*|}\right).
\end{align*}
The utility of Eq. (\ref{eq:restrictedER}) and (\ref{eq:restrictedBip}) in the present $\rho$-SBM setting can be realized as follows.  For each $i,j\in[K]$, we define
$\xi_{i,j}:=1
+\frac{\rho}{(1-\Lambda_{i,j})\Lambda_{i,j}(1-\rho)^2}>1.$
Combining the above yields (where $A_1$ is the adjacency matrix of $g_1$ and $B_2$ the adjacency matrix of $g_2$, and for $j,\ell\in[K]$ $A_1^j$ is $A_1$ restricted to $g_1^j$, and $A_1^{j,\ell}$ is $A_!$ restricted to $g_1^{j,\ell}$; similarly for $B_2$)
\begin{align*}
I&(G_1;G_2)-I(G_1;\boldsymbol{\sigma}^*(G_2))\\
&=-\sum_{\substack{g_1\in \mathcal{G}_n,\\ g_2\in \mathcal{G}_n}}\p(g_1,g_2)
\log\Bigg(\sum_{\tau\in \Pi(n)^*}\frac{1}{|\Pi(n)^*|}\prod_{j\in[K]} 
\text{exp}\left\{\log(\xi_{j,j})
\frac{1}{4}\left(\|A_1^j-B_2^j\|_F^2-\|A_1^j-P_{\tau_j} B_2^jP_{\tau_j}^T\|_F^2\right)  \right\}\notag\\
&\hspace{20mm}
\prod_{\substack{j,\ell\in[K]\\\text{ s.t. }j<\ell}} \text{exp}\left\{\log(\xi_{j,\ell})
\frac{1}{4}\left(\|A_1^{j,\ell}-B_2^{j,\ell}\|_F^2-\|A_1^{j,\ell}-P_{\tau_j} B_2^{j,\ell}P_{\tau_\ell}^T\|_F^2 \right) \right\}\Bigg)
\end{align*}
If $\min_i n_i=\Theta(n)$ and $\rho=\omega\left(\sqrt{\frac{\log n}{n}}\right)$, for $n$ sufficiently large there exists constants $C>0$, $C'>0$ such that for all $\tau\in \Pi(n)^*$,
$$-\frac{1}{4}\left(\|A_1^{j,\ell}-P_{\tau_j} B_2^{j,\ell}P_{\tau_\ell}^T\|_F^2-\|A_1^{j,\ell}-B_2^{j,\ell}\|_F^2 \right)<-C n  \rho$$
and 
$$-\frac{1}{4}\left(\|A_1^{j}-P_{\tau_j} B_2^{j}P_{\tau_j}^T\|_F^2-\|A_1^{j}-B_2^{j}\|_F^2\right)<-C n  \rho$$
with probability at least $1-C'e^{-3\log n}$.
Therefore, for $n$ sufficiently large there exists a constant $C''>0$ such that 
\begin{align}
\label{eq:infolossSstar}
I&(G_1;G_2)-I(G_1;\boldsymbol{\sigma}^*(G_2))\geq
C''n\rho^2- C'e^{-3\log n}n^2=\Omega(n\rho^2)
\end{align}
as desired.
The proof is then completed by applying straightforward application of the information processing inequality which yields the following proposition. 
\begin{proposition}
\label{prop:ipi2}
Let $\boldsymbol{\sigma}$ be uniformly distributed on $\Pi(n)$ independent of $\boldsymbol{\sigma}^*$ uniformly distributed on $\Pi(n)^*$.
With notation as above, 
$I(G_1;\boldsymbol{\sigma}(G_2))\leq I(G_1;\boldsymbol{\sigma}^*(G_2)).$ 
\end{proposition}
\noindent The proof of Proposition \ref{prop:ipi2} is straightforward and hence is omitted.

\subsection{Proof of Theorem \ref{them:GMSBM}}
\label{sec:proofSBMGM}

Herein, we prove Theorem \ref{them:GMSBM}.
With notation and assumptions as in the Theorem, for $\tau\in \Pi(n)$, define
$X_{\tau,A,B}:=\frac{1}{2}(\|A- \pt B\pt^T\|_F^2-\|A- B\|_F^2).$
Fix $\tau\neq \text{id}_n\in \Pi(n)$, and suppose that $\tau$ permutes the labels of exactly $m\geq2$ vertices (so that $|\{v:\tau(v)=v\}|=n-m$). 
For each pair $1\leq i,j\leq K$, let
$$\epsilon_{i,j}^{\tau}:=\big|\left\{ v\in V_i\text{ s.t. }\tau(v)\in V_j,\, v\neq \tau(v)\right\}\big|,$$
and let 
$$f^{\tau}_i=\big|\left\{ v\in V_i\text{ s.t. }\tau(v)=v\right\}\big|.$$
Note that for each $i\in[K]$ and each $\tau\in \Pi(n)$, we have that
$$n_i-\fo_i=\sum_{j}\eo_{i,j}=\sum_{j}\eo_{j,i}.$$

As in the proof of Lemma \ref{lem:preludeto10}, we note that if $x,y\in \mathcal{G}_n$ with adjacency matrices $A_x$ and $B_y$ respectively,
$
X_{\tau,A_x,B_y}=\frac{1}{2}\|A_x-\pt A_x\pt^T\|_F^2-2F_\mathcal{A}-2F_\mathcal{O},
$
where $F_\mathcal{A}:=F_\mathcal{A}(x,y,\tau)$ and $F_\mathcal{O}:=F_\mathcal{O}(x,y,\tau)$ are defined as in Eq. (\ref{eq:Fa})--(\ref{eq:Fo}).
We call the errors induced by $\tau$ on $x$ of the form $u\nsim_{x}v\text{ and }u\sim_{\tau(x)}v$ \textit{addition errors} (so that $F_\mathcal{A}$ is the number of fixed addition errors), and the errors induced by $\tau$ on $x$ of the form $u\sim_{x}v\text{ and }u\nsim_{\tau(x)}v$ \textit{occlusion errors} (so that $F_\mathcal{O}$ is the number of fixed occlusion errors).

We will first show that if $(G_1,G_2)\sim\rho-$SBM($K,\vec{n},b,\Lambda$) satisfying the assumptions in the theorem, then with sufficiently high probability $(G_1,G_2)=(x,y)$ satisfying
$$\frac{1}{2}\|A_x-\pt A_x\pt^T\|_F^2>2F_\mathcal{A}(x,y,\tau)+2F_\mathcal{O}(x,y,\tau),$$
implying that $X_{\tau,A_x,B_y}>0$.
To this end, with $A$ and $B$ the random adjacency matrices associated with $G_1$ and $G_2$ respectively, note that
\begin{align}
\frac{1}{2}\|A-\pt A^T \pt^T\|_F^2&=\sum_{\{v,v'\}\in\binom{V}{2}}
(A_{v,v'}-A_{\tau(v),\tau(v')})^2\notag\\
\label{eq:withinblock}
&=\sum_{i=1}^K\sum_{\{v,v'\}\in \binom{V_i}{2}}A_{v,v'}(1-A_{\tau(v),\tau(v')})+(1-A_{v,v'})A_{\tau(v),\tau(v')}\\
\label{eq:acrossblock}
&\hspace{5mm}+\sum_{i=1}^K\sum_{j>i}^K\sum_{(v,v')\in V_i\times V_j}A_{v,v'}(1-A_{\tau(v),\tau(v')})+(1-A_{v,v'})A_{\tau(v),\tau(v')}.
\end{align}
Consider the sum in (\ref{eq:withinblock}). 
For each $i\in[K]$, the sum can be further decomposed into three terms:
\begin{itemize}
\item[1.]For each $j\in[K]$ there are 
$\mathfrak{n}_1(i,i,j,j):=\binom{\epsilon^{\tau}_{i,j}}{2}$ terms with both $v\neq \tau(v)$ and $v'\neq \tau(v')$ mapped to $V_j$ by $\tau.$  
The expected number of addition errors (denoted $\mathcal{A}^{(1)}_{i,i,j,j})$ contributed by these terms is
$$\e(\mathcal{A}^{(1)}_{i,i,j,j})=\binom{\epsilon^{\tau}_{i,j}}{2}(1-\Lambda_{i,i})\Lambda_{j,j}  ,$$ 
and the expected number of occlusion errors (denoted $\mathcal{O}^{(1)}_{i,i,j,j}$) contributed by these terms is
$$\e(\mathcal{O}^{(1)}_{i,i,j,j})=\binom{\epsilon^{\tau}_{i,j}}{2}\Lambda_{i,i}(1-\Lambda_{j,j}) .$$ 
Conditioning on $\mathcal{A}^{(1)}_{i,i,j,j}$ and $\mathcal{O}^{(1)}_{i,i,j,j}$, each addition error is independently corrected in $B$ with probability $(1-\rho)(1-\Lambda_{j,j}),$
and each occlusion error is independently corrected (independently also of the corrected addition errors) in $B$ with probability $(1-\rho)\Lambda_{j,j}.$
\item[2.]$\text{ For each }j\in[K],\,\ell\in[K],\,\ell>j, \text{ there are }\mathfrak{n}_2(i,i,j,\ell):=\epsilon^{\tau}_{i,j}\epsilon^{\tau}_{i,\ell}\text{ terms}\text{ with }v\neq \tau(v)\in V_j\text{ and }v'\neq \tau(v')\in V_\ell.$
The expected number of addition errors (denoted $\mathcal{A}^{(2)}_{i,i,j,\ell})$ contributed by these terms is
$$\e(\mathcal{A}^{(2)}_{i,i,j,\ell})=\epsilon^{\tau}_{i,j}\epsilon^{\tau}_{i,\ell}(1-\Lambda_{i,i})\Lambda_{j,\ell}  ,$$ 
and the expected number of occlusion errors (denoted $\mathcal{O}^{(2)}_{i,i,j,\ell})$contributed by these terms is
$$\e(\mathcal{O}^{(2)}_{i,i,j,\ell})=\epsilon^{\tau}_{i,j}\epsilon^{\tau}_{i,\ell}\Lambda_{i,i}(1-\Lambda_{j,\ell}) .$$
Conditioning on $\mathcal{A}^{(2)}_{i,i,j,\ell}$ and $\mathcal{O}^{(2)}_{i,i,j,\ell}$, each addition error is independently corrected in $B$ with probability $(1-\rho)(1-\Lambda_{j,\ell}),$
and each occlusion error is independently corrected (independently also of the corrected addition errors) in $B$ with probability $(1-\rho)\Lambda_{j,\ell}.$
\item[3.]$\text{ For each }j\in[K], \text{ there are }\mathfrak{n}_3(i,i,i,j):=\fo_i \epsilon^{\tau}_{i,j}\text{ terms}\text{ with }v= \tau(v)\in V_i\text{ and }v'\neq \tau(v')\in V_j.$
The expected number of addition errors (denoted $\mathcal{A}^{(3)}_{i,i,i,j})$ contributed by these terms is
$$\e(\mathcal{A}^{(3)}_{i,i,i,j})=\fo_i \epsilon^{\tau}_{i,j}(1-\Lambda_{i,i})\Lambda_{i,j}  ,$$ 
and the expected number of occlusion errors (denoted $\mathcal{O}^{(3)}_{i,i,i,j})$ contributed by these terms is
$$\e(\mathcal{O}^{(3)}_{i,i,i,j})=\fo_i \epsilon^{\tau}_{i,j}\Lambda_{i,i}(1-\Lambda_{i,j}) .$$
Conditioning on $\mathcal{A}^{(3)}_{i,i,i,j}$ and $\mathcal{O}^{(3)}_{i,i,i,j}$, each addition error is independently corrected in $B$ with probability $(1-\rho)(1-\Lambda_{i,j}),$
and each occlusion error is independently corrected (independently also of the corrected addition errors) in $B$ with probability $(1-\rho)\Lambda_{i,j}.$
\end{itemize}
In the sum in (\ref{eq:acrossblock}), for each $i,j\in[K]$ with $j>i$, the sum can be further decomposed into three terms:
\begin{itemize}
\item[4.]$\text{ For each }\ell\in[K],\,h\in[K], \text{ there are }\mathfrak{n}_4(i,j,h,\ell):=\epsilon^{\tau}_{i,h}\epsilon^{\tau}_{j,\ell}\text{ terms}\text{ with }v\in V_i,\,v\neq \tau(v)\in V_h\text{ and }v'\in V_j,\,v'\neq \tau(v')\in V_\ell.$
The expected number of addition errors (denoted $\mathcal{A}^{(4)}_{i,j,h,\ell})$ contributed by these terms is
$$\e(\mathcal{A}^{(4)}_{i,j,h,\ell})=\epsilon^{\tau}_{i,h}\epsilon^{\tau}_{j,\ell}(1-\Lambda_{i,j})\Lambda_{h,\ell}  ,$$ 
and the expected number of occlusion errors (denoted $\mathcal{O}^{(4)}_{i,j,h,\ell})$ contributed by these terms is
$$\e(\mathcal{O}^{(4)}_{i,j,h,\ell})=\epsilon^{\tau}_{i,h}\epsilon^{\tau}_{j,\ell}\Lambda_{i,j}(1-\Lambda_{h,\ell}) .$$
Conditioning on $\mathcal{A}^{(4)}_{i,j,h,\ell}$ and $\mathcal{O}^{(4)}_{i,j,h,\ell}$, each addition error is independently corrected in $B$ with probability $(1-\rho)(1-\Lambda_{h,\ell}),$
and each occlusion error is independently corrected (independently also of the corrected addition errors) in $B$ with probability $(1-\rho)\Lambda_{h,\ell}.$
\item[5.]$\text{ For each }\ell\in[K],\, \text{ there are }\mathfrak{n}_5(i,j,i,\ell):=\fo_i \epsilon^{\tau}_{j,\ell}\text{ terms}\text{ with }v= \tau(v)\in V_i\text{ and }v'\in V_j,\, v'\neq \tau(v')\in V_\ell.$
The expected number of addition errors (denoted $\mathcal{A}^{(5)}_{i,j,i,\ell})$ contributed by these terms is
$$\e(\mathcal{A}^{(5)}_{i,j,i,\ell})=\fo_i \epsilon^{\tau}_{j,\ell}(1-\Lambda_{i,j})\Lambda_{i,\ell}  ,$$ 
and the expected number of occlusion errors (denoted $\mathcal{O}^{(5)}_{i,j,i,\ell})$  contributed by these terms is
$$\e(\mathcal{O}^{(5)}_{i,j,i,\ell})=\fo_i \epsilon^{\tau}_{j,\ell}\Lambda_{i,j}(1-\Lambda_{i,\ell}) .$$
Conditioning on $\mathcal{A}^{(5)}_{i,j,i,\ell}$ and $\mathcal{O}^{(5)}_{i,j,i,\ell}$, each addition error is independently corrected in $B$ with probability $(1-\rho)(1-\Lambda_{i,\ell}),$
and each occlusion error is independently corrected (independently also of the corrected addition errors) in $B$ with probability $(1-\rho)\Lambda_{i,\ell}.$
\item[6.]$\text{ For each }\ell\in[K],\, \text{ there are }\mathfrak{n}_6(i,j,\ell,j):=\fo_j \epsilon^{\tau}_{i,\ell}\text{ terms}\text{ with }v'= \tau(v')\in V_j\text{ and }v\in V_i,\, v\neq \tau(v)\in V_\ell.$
The expected number of addition errors (denoted $\mathcal{A}^{(6)}_{i,j,\ell,j})$ contributed by these terms is
$$\e(\mathcal{A}^{(6)}_{i,j,\ell,j})=\fo_j \epsilon^{\tau}_{i,\ell}(1-\Lambda_{i,j})\Lambda_{\ell,j}  ,$$ 
and the expected number of occlusion errors (denoted $\mathcal{O}^{(6)}_{i,j,\ell,j})$ contributed by these terms is
$$\e(\mathcal{O}^{(6)}_{i,j,\ell,j})=\fo_j \epsilon^{\tau}_{i,\ell}\Lambda_{i,j}(1-\Lambda_{\ell,j}) .$$
Conditioning on $\mathcal{A}^{(6)}_{i,j,\ell,j}$ and $\mathcal{O}^{(6)}_{i,j,\ell,j}$, each addition error is independently corrected in $B$ with probability $(1-\rho)(1-\Lambda_{\ell,j}),$
and each occlusion error is independently corrected (independently also of the corrected addition errors) in $B$ with probability $(1-\rho)\Lambda_{\ell,j}.$
\end{itemize}

For each $s\in[6]$ and each feasible set of indices $(a,b,c,d)\in K^4$, note that $\mathfrak{n}_s(a,b,c,d)=O(mn)$.   
If $\mathfrak{n}_s(a,b,c,d)\geq m \sqrt{n \log n}$, 
then an application of \cite[Proposition 3.2]{kim} yields that there exists a constant $\gamma>0$ such that for $n$ sufficiently large
\begin{equation}
\label{eq:bigadd}
\p\left(|\mathcal{A}^{(s)}_{a,b,c,d}-\e \mathcal{A}^{(s)}_{a,b,c,d}|>\gamma \sqrt{\mathfrak{n}_s(a,b,c,d)}\sqrt{m\log n}   \right)\leq 2e^{-3 m\log n},
\end{equation}
and
\begin{equation}
\label{eq:bigocc}
\p\left(|\mathcal{O}^{(s)}_{a,b,c,d}-\e \mathcal{O}^{(s)}_{a,b,c,d}|>\gamma \sqrt{\mathfrak{n}_s(a,b,c,d)}\sqrt{m\log n}   \right)\leq 2e^{-3 m\log n}.
\end{equation}
Alternatively, if $\mathfrak{n}_s(a,b,c,d)<m\sqrt{n\log n}$, then it is immediate that 
\begin{equation}
\label{eq:small}
|\mathcal{A}^{(s)}_{a,b,c,d}-\e \mathcal{A}^{(s)}_{a,b,c,d}|\leq m\sqrt{n\log n},\text{ and  }|\mathcal{O}^{(s)}_{a,b,c,d}-\e \mathcal{O}^{(s)}_{a,b,c,d}|\leq m\sqrt{n\log n}.
\end{equation}
Denote by $\mathcal E$ the event that $A$ satisfies
\begin{align*}
&|\mathcal{A}^{(s)}_{a,b,c,d}-\e \mathcal{A}^{(s)}_{a,b,c,d}|\leq\gamma \sqrt{\mathfrak{n}_s(a,b,c,d)}\sqrt{m\log n},\\
&|\mathcal{O}^{(s)}_{a,b,c,d}-\e \mathcal{O}^{(s)}_{a,b,c,d}|\leq\gamma \sqrt{\mathfrak{n}_s(a,b,c,d)}\sqrt{m\log n} 
\end{align*}
for all $s\in[6]$ and feasible indices $(a,b,c,d)\in K^4$ satisfying $\mathfrak{n}_s(a,b,c,d)\geq m \sqrt{n \log n}$.
Eq. (\ref{eq:bigadd})--(\ref{eq:bigocc}) and a simple union bound imply that 
\begin{equation}
\label{eq:boundonE}
\p(\mathcal E^c)\leq 24K^4 e^{-3 m\log n}.
\end{equation} 
At each sample point $\varepsilon$ in $\mathcal{E}$, recalling $\mathfrak{n}_s(a,b,c,d)=O(mn)$, we have that there exists a constant $C>0$ (which can be chosen independent of $\varepsilon$) such that 
\begin{align*}
\left|\frac{\|A-\pt A^T \pt^T\|_F^2}{2}-\e\left(\frac{\|A-\pt A^T\pt ^T\|_F^2}{2}\right)\right|\leq Cm\sqrt{n\log n}.
\end{align*}
A brief calculation then yields
\begin{align}
\label{eq:bigugly}
&\e\left(F_\mathcal{A}+F_\mathcal{O}\,\,\bigg|\,\,\varepsilon\right)\leq(1-\rho)\Bigg[\sum_{i=1}^K\sum_{j=1}^K \binom{\epsilon^{\tau}_{i,j}}{2}\Lambda_{j,j}(1-\Lambda_{j,j})+\sum_{i=1}^K\sum_{j=1}^K\sum_{\ell>j}^K \epsilon^{\tau}_{i,j}\epsilon^{\tau}_{i,\ell}\Lambda_{j,\ell}(1-\Lambda_{j,\ell})\notag\\
&\hspace{5mm}+\sum_{i=1}^K\sum_{j=1}^K \fo_i \epsilon^{\tau}_{i,j}\Lambda_{i,j}(1-\Lambda_{i,j})+\sum_{i=1}^K\sum_{j>i}^K\sum_{\ell=1}^K\sum_{h=1}^K \epsilon^{\tau}_{i,h}\epsilon^{\tau}_{j,\ell}\Lambda_{h,\ell}(1-\Lambda_{h,\ell})\notag\\
&\hspace{5mm}+\sum_{i=1}^K\sum_{j>i}^K\sum_{\ell=1}^K \fo_i \epsilon^{\tau}_{j,\ell}\Lambda_{i,\ell}(1-\Lambda_{i,\ell})+\sum_{i=1}^K\sum_{j>i}^K\sum_{\ell=1}^K \fo_j \epsilon^{\tau}_{i,\ell}\Lambda_{\ell,j}(1-\Lambda_{\ell,j})\Bigg]+Cm\sqrt{n\log n},
\end{align}
where conditioning on $\varepsilon$ is understood to mean conditioning on a value of $A$ that satisfies the conditions imposed by $\mathcal{E}$.
Note that the term in the brackets in Eq. (\ref{eq:bigugly}) is equal to
\begin{align}
&\sum_{i=1}^K\sum_{j=1}^K \binom{\epsilon^{\tau}_{i,j}}{2}\Lambda_{i,i}(1-\Lambda_{i,i})+\sum_{i=1}^K\sum_{j=1}^K\sum_{\ell>j}^K \epsilon^{\tau}_{i,j}\epsilon^{\tau}_{i,\ell}\Lambda_{i,i}(1-\Lambda_{i,i})+\sum_{i=1}^K\sum_{j=1}^K \fo_i \epsilon^{\tau}_{i,j}\Lambda_{i,i}(1-\Lambda_{i,i})\notag\\
&\hspace{10mm}+\sum_{i=1}^K\sum_{j>i}^K\sum_{\ell=1}^K\sum_{h=1}^K \epsilon^{\tau}_{i,h}\epsilon^{\tau}_{j,\ell}\Lambda_{i,j}(1-\Lambda_{i,j})+\sum_{i=1}^K\sum_{j>i}^K\sum_{\ell=1}^K \fo_i \epsilon^{\tau}_{j,\ell}\Lambda_{i,j}(1-\Lambda_{i,j})\notag\\
\label{eq:equiv2}
&\hspace{10mm}+\sum_{i=1}^K\sum_{j>i}^K\sum_{\ell=1}^K \fo_j \epsilon^{\tau}_{i,\ell}\Lambda_{j,i}(1-\Lambda_{j,i}).
\end{align}
Also note that for any indices $i,j,k,\ell\in[K]$,
\begin{align}
\label{eq:observe}
\Lambda_{i,j}(1-\Lambda_{i,j})+\Lambda_{k,\ell}(1-\Lambda_{k,\ell})\leq
\Lambda_{i,j}(1-\Lambda_{k,\ell})+\Lambda_{k,\ell}(1-\Lambda_{i,j}).\end{align}
Combining Eqs. (\ref{eq:bigugly}) and (\ref{eq:equiv2}) with Eq. (\ref{eq:observe}) then yields that for $\varepsilon\in\mathcal{E}$
\begin{align*}
\e\left(F_\mathcal{A}+F_\mathcal{O}\,\,|\,\,\varepsilon\right)
\leq \frac{1-\rho}{2} \e\left(\frac{1}{2}\|A-\pt A^T\pt^T\|_F^2\right)+C m\sqrt{n\log n}.
\end{align*}
%
Applying Hoeffding's inequality to 
$F:=F_\mathcal{A}+F_\mathcal{O}$
yields that there exists constants $c_1>0$, $c_2$, and $c_3$, such that for $n$ sufficiently large
\begin{align*}
\p\left(2F\geq \frac{1}{2}\|A-\pt A^T\pt^T\|_F^2\,\,\Big|\,\,\varepsilon\right)
&= \p\left(2F-2\e (F|\varepsilon)\geq \frac{1}{2}\|A-\pt A^T\pt^T\|_F^2-2\e  (F|\varepsilon)\,\,\Big|\,\,\varepsilon\right)\\
&\leq \p\left(F-\e  (F|\varepsilon)\geq \frac{\rho}{4}\e\left(\|A-\pt A^T\pt^T\|_F^2\right)\!-\!3Cm\sqrt{n\log n}\,\,\Big|\,\,\varepsilon\right)\\
&\leq \text{exp}\bigg\{ -c_1\rho^2 mn+c_2\rho m\sqrt{n\log n}+c_3 m\log n\bigg\},
\end{align*}
where the last inequality follows from the fact that that $F$ is the sum of 
at most $\binom{m}{2}+m(n-m)$
independent Bernoulli random variables.  
Therefore, there exists a constant $\alpha>0$ such that if $\rho>\sqrt{\alpha\frac{\log n}{n}}$ then 
$$\p\left(2F\geq \frac{1}{2}\|A-\pt A^T\pt^T\|_F^2\,\,\bigg|\,\,\varepsilon\right)\leq e^{-3 m\log n}.$$
Combined with $\p(\mathcal E^c)\leq24K^4 e^{-3 m\log n}$, we have that, unconditionally,
$$\p\left(2F\geq \frac{1}{2}\|A-\pt A^T\pt^T\|_F^2\right)\leq (24K^4+1) e^{-3 m\log n}.$$
Summing over $\tau$ and $n$ yields that
\begin{align*}
\p\big(\exists\, \tau\in \Pi(n) \text{ with }X_{\tau,A,B}\leq -1  \big)
&\leq (24K^4+1)  e^{- 3\log n},
\end{align*}
as desired.


\subsection{Proof of Theorem \ref{thm:no_match}}
\label{sec:suppthm2}
In this section, we will prove Theorem \ref{thm:no_match}.  
Recall that for $\tau\in \Pi(n)$, we define
$X_{\tau,A,B}:=\frac{1}{2}(\|A- \pt B\pt^T\|_F^2-\|A- B\|_F^2).$
The proof of Theorem \ref{thm:no_match} will proceed as follows. 
For each $i=1,2,\ldots,N$, let $E_{\tau_i,A,B}$ be the event $\{X_{\tau_i,A,B}\leq -1 \}$, and let $X=\sum_{i}\mathds{1}\{E_{\tau_i,A,B}\}$.
We first show in Lemma \ref{lem:lowerboundonp} that for judiciously chosen $\beta$,
$$\e(X)=\Omega\left( \frac{N}{n^{1/5}\sqrt{\log n}}\right).$$
In Lemmas \ref{lem:covbnd}--\ref{lem:varbnd}, we show that 
$$\text{Var}(X)=O\left(\frac{N^2}{\sqrt{n}}  \right).$$
As $X$ is a nonnegative integer-valued random variable, we apply the second moment method \cite[Theorem 4.3.1]{alon2015probabilistic} to derive
$
\p(X=0)\leq\frac{\text{Var}(X)}{\e(X)^2}=O\left( \frac{\sqrt{\log n}}{n^{1/10}}\right)=o(1)$
as desired.
We now establish the supporting Lemmas \ref{lem:lowerboundonp}--\ref{lem:varbnd}.
We begin by noting that if $\tau=i\leftrightarrow j$ is a within-block transposition, i.e., $b(i)=b(\tau(i))=b(j)=b(\tau(j)),$
then
\begin{align}
X_{\tau}:=X_{\tau,A,B}:&=\|A- \po B\po^T\|_F^2-\|A- B\|_F^2=
2\sum_{\substack{\ell,k\text{ s.t. }\tau(\ell)\neq \ell\\\text{ or }\tau(k)\neq k}}\left(A_{\ell,k}B_{\ell,k}-A_{\ell,k}B_{\tau(\ell),\tau(k)}\right)\notag\\
&=4\sum_{k\neq i,j}(A_{i,k}B_{i,k}-A_{i,k}B_{j,k})+4\sum_{k\neq i,j}(A_{j,k}B_{j,k}-A_{j,k}B_{i,k})\notag
\end{align} 
From this, we immediately arrive at 
\begin{align}
\label{eq:errortrans}
X_{\tau}&=4\sum_{k\neq i,j}(A_{i,k}-A_{j,k})(B_{i,k}-B_{j,k}). 
\end{align}
For each $k\neq i,j$, the terms $X_{\tau}^{(k)}=X_{\tau,A,B}^{(k)}:=4(A_{i,k}-A_{j,k})(B_{i,k}-B_{j,k})$ in Eq. (\ref{eq:errortrans}) are independent
with mean
$$\mu_k:=\e\left(4(A_{i,k}-A_{j,k})(B_{i,k}-B_{j,k})\right)=8\Lambda_{b(i),b(k)}(1-\Lambda_{b(i),b(k)})\rho,$$ 
and variance
\begin{align*}
\sigma^2_k&=\text{Var}\left(4(A_{i,k}-A_{j,k})(B_{i,k}-B_{j,k})\right)\\
&=16\big(2\Lambda_{b(i),b(k)}(1-\Lambda_{b(i),b(k)})\rho+4\Lambda_{b(i),b(k)}^2(1-\Lambda_{b(i),b(k)})^2+8(1-\Lambda_{b(i),b(k)})\Lambda_{b(i),b(k)}^3\rho\\
&\hspace{5mm}-8\Lambda_{b(i),b(k)}^2(1-\Lambda_{b(i),b(k)})\rho+4\rho^2\Lambda_{b(i),b(k)}^2(1-\Lambda_{b(i),b(k)})^2\big)
-\left(8\Lambda_{b(i),b(k)}(1-\Lambda_{b(i),b(k)})\rho\right)^2.
\end{align*}
Note here that for all $k$,
$$\lim_{\rho\rightarrow 0}\mu_k=0,\hspace{5mm} \lim_{\rho\rightarrow 0}\sigma^2_k=48\Lambda_{b(i),b(k)}^2(1-\Lambda_{b(i),b(k)})^2>0.$$
Next, note that
$
\xi_k:=\e\left(\big|4(A_{i,k}-A_{j,k})(B_{i,k}-B_{j,k})-\mu_k\big|^3\right)\leq 8^3
$  
and is bounded away from $0$ as $\rho\rightarrow 0$.
We define
$$\mu_\tau=\sum_{k\neq i,j}\mu_k,\hspace{3mm} \sigma^2_\tau=\sum_{k\neq i,j}\sigma^2_k,\hspace{3mm} \xi_\tau=\sum_{k\neq i,j}\xi_k.$$
The classic Berry-Esseen theorem \cite[Theorem XVI.5.2]{feller2008introduction} yields
$$\sup_{x}\left|\p\bigg(\frac{X_{\tau}-\mu_\tau}{\sigma_\tau} \leq x\bigg)-\Phi(x)   \right|\leq \frac{6\xi_\tau}{\sigma_\tau^3},$$
where $\Phi(\cdot)$ is the standard normal cumulative distribution function.
Using the inequality \cite[Eq. 7.1.13]{abramowitz1964handbook}
$$\frac{1}{x+\sqrt{x^2+2}}e^{-x^2}<\int_{x}^{\infty}e^{-t^2}dt\leq \frac{1}{x+\sqrt{x^2+\frac{4}{\pi}}}e^{-x^2}\text{ for all }x\geq0,$$
which is equivalent to
\begin{equation}
\label{eq:gaussianbounds}
\frac{1}{\sqrt{\pi}\left(x+\sqrt{x^2+2}\right)}e^{-x^2}<1-\Phi\left(\sqrt{2} x\right)\leq \frac{1}{\sqrt{\pi}\left(x+\sqrt{x^2+\frac{4}{\pi}}\right)}e^{-x^2} \text{ for all }x\geq0,
\end{equation}
we have that (with $x:=\frac{1+\mu_\tau}{\sqrt{2}\sigma_\tau})$
\begin{align}
\label{eq:bebound1}
\p(X_{\tau}\leq -1)=\p\left(\frac{X_{\tau}-\mu_\tau}{\sigma_\tau}\leq -\sqrt{2} x\right)&\leq 
\frac{1}{\sqrt{\pi}\left(x+\sqrt{x^2+\frac{4}{\pi}}\right)}e^{-x^2}+\frac{6\xi_\tau}{\sigma_\tau^3},
\end{align}
and
\begin{align}
\label{eq:bebound2}
\p(X_{\tau}\leq -1)=\p\left(\frac{X_{\tau}-\mu_\tau}{\sigma_\tau}\leq -\sqrt{2} x\right)
&\geq \frac{1}{\sqrt{\pi}\left(x+\sqrt{x^2+2}\right)}e^{-x^2}-\frac{6\xi_\tau}{\sigma_\tau^3}.
\end{align}
\begin{lemma}
\label{lem:lowerboundonp}
With notation as above, there exists a constant $\beta$ such that if $\rho\leq \sqrt{\frac{\beta\log n}{n}}$,
then 
$\p(X_{\tau}\leq -1)=\Omega\left( \frac{1}{n^{1/4}\sqrt{\log n}}\right)$
\end{lemma}
\begin{proof}
We first note that if $\rho\leq \sqrt{\frac{\beta\log n}{n}}$, then there exists a constant $c_1>0$ such that for $n$ sufficiently large, $\sigma^2_\tau\geq c_1 n$.
Therefore, 
\begin{align}
\label{boundxhigh}
x&=\frac{1+\mu_\tau}{\sqrt{2}\sigma_\tau}\leq \frac{1+2\rho n}{\sqrt{2c_1n}}
\leq \frac{1+2\sqrt{\beta n \log n}}{\sqrt{2c_1n}}
=\sqrt{\frac{2\beta\log n}{c_1}} +\frac{1}{\sqrt{2c_1n}},
\end{align}
and $\beta$ can be chosen so that $\rho\leq \sqrt{\frac{\beta\log n}{n}}$ implies that $x\leq \sqrt{\frac{\log n}{5}}$.
With this choice of $\beta$, 
The lower bound in Eq. (\ref{eq:bebound2}) is then bounded by
\begin{align*}
\frac{1}{\sqrt{\pi}\left(x+\sqrt{x^2+2}\right)}e^{-x^2}-\frac{6\xi_\tau}{\sigma_\tau^3}\geq\frac{\text{exp}\left\{ -\frac{\log n}{5}\right\}}{\sqrt{\pi}\left(\sqrt{\frac{\log n}{5}}+\sqrt{\frac{\log n}{5}+2}\right)}-\Theta\left( n^{-1/2}\right)
=\Omega\left( \frac{1}{n^{1/5}\sqrt{\log n}}\right),
\end{align*}
as desired.
\end{proof}
Recalling that $E_{\tau_i}=E_{\tau_i,A,B}$ is the event $\{X_{\tau_i}\leq -1 \},$ Lemma \ref{lem:lowerboundonp} is equivalent to 
$$\e(\mathds{1}\{E_{\tau_i}\})=\p(E_{\tau_i})=\p(X_{\tau_i}\leq -1)=\Omega\left( \frac{1}{n^{1/5}\sqrt{\log n}}\right).$$
It follows immediately that (where $N$ is as defined in Theorem \ref{thm:no_match})
$$\e(X)=\Omega\left( \frac{N}{n^{1/5}\sqrt{\log n}}\right).$$
We now turn our attention to bounding $\text{Var}(X).$
\begin{lemma}
\label{lem:covbnd}
With notation as above and assumptions as in Theorem \ref{thm:no_match}, let $\tau_1=i\leftrightarrow j$ and $\tau_2=h\leftrightarrow\ell$ be two disjoint, within-block transpositions.  
There exists constants $C_1>0$ and $C_2>0$,
\begin{align*}
\text{Cov}(\mathds{1}\{E_{\tau_i}\},\mathds{1}\{E_{\tau_j}\})
&\leq C_2\left(\text{exp}\left\{-C_1\rho^2 n\right\}
 +\Theta\left( \frac{1}{\sqrt{n}}\right)\right)
 \Theta\left( \frac{1}{\sqrt{n}}\right)
\end{align*}
\end{lemma}
\begin{proof}
Note
\begin{align*}
\text{Cov}(\mathds{1}\{E_{\tau_1}\},\mathds{1}\{E_{\tau_2}\})&=\p(X_{\tau_1}\leq -1,X_{\tau_2}\leq -1)-\p(X_{\tau_1}\leq -1)\p(X_{\tau_2}\leq -1)
\end{align*}
Observe that $X_{\tau_1}$ and $X_{\tau_2}$ are each then the sum of $n-2$ independent terms ($\{X^{(k)}_{\tau_1}\}_{k\neq i,j}$ and $\{X^{(k)}_{\tau_2}\}_{k\neq h,\ell}$ resp.) which are collectively independent except for the four terms
$$X^{(h)}_{\tau_1}=4(A_{i,h}-A_{j,h})(B_{i,h}-B_{j,h}),\,X^{(\ell)}_{\tau_1}=4(A_{i,\ell}-A_{j,\ell})(B_{i,\ell}-B_{j,\ell}),$$
$$X^{(i)}_{\tau_2}=4(A_{h,i}-A_{\ell,i})(B_{h,i}-B_{\ell,i}),\,X^{(j)}_{\tau_2}=4(A_{h,j}-A_{\ell,j})(B_{h,j}-B_{\ell,j}).$$
Let $\widetilde X_{\tau_2}=X_{\tau_2}-X^{(i)}_{\tau_2}-X^{(j)}_{\tau_2}$, so that $X_{\tau_1}$ and $\widetilde X_{\tau_2}$ are independent.
Noting that $|\widetilde X_{\tau_2}-X_{\tau_2}|\leq 8$, we have 
$$\p(X_{\tau_1}\leq -1,X_{\tau_2}\leq -1)\leq\p(X_{\tau_1}\leq -1,\widetilde X_{\tau_2}\leq 7)=\p(X_{\tau_i}\leq -1)\p(\widetilde X_{\tau_j}\leq 7).$$
Therefore,
\begin{align}
\label{eq:cov}
\text{Cov}(\mathds{1}\{E_{\tau_1}\},\mathds{1}\{E_{\tau_2}\})&\leq \p(X_{\tau_1}\leq -1)\left(\p(\widetilde X_{\tau_2}\leq 7)-\p(X_{\tau_2}\leq -1)\right)\notag\\
&\leq \p(X_{\tau_1}\leq -1)\left(\p(\widetilde X_{\tau_2}\leq 7)-\p(\widetilde X_{\tau_2}\leq -9)\right)
\end{align}
As in Eq. (\ref{boundxhigh}), there is a constant $c_2>0$ such that $\sigma_\tau^2\leq c_2n$ for $n$ sufficiently large, and
\begin{align}
\label{boundxlow}
x_1&:=\frac{1+\mu_{\tau_1}}{\sqrt{2}\sigma_{\tau_1}}\geq \frac{1+8\eta(1-\eta)\rho n}{\sqrt{2c_2n}}\geq \frac{8\eta(1-\eta)\rho n}{\sqrt{2c_2n}}=\frac{8\eta(1-\eta)\rho\sqrt{n}}{\sqrt{2c_2}},
\end{align}
and from Eq. (\ref{eq:bebound1}) we have that
\begin{align}
\label{eq:cov1}
\p(X_{\tau_1}\leq -1)&\leq
\frac{\text{exp}\left\{-\left(\frac{8\eta(1-\eta)\rho \sqrt{n}}{\sqrt{2c_2}}\right)^2\right\}}{2}
 +\Theta\left( \frac{1}{\sqrt{n}}\right)\notag\\
 &=\frac{1}{2}\text{exp}\left\{-C_1\rho^2 n\right\}
 +\Theta\left( \frac{1}{\sqrt{n}}\right).
\end{align}
for a constant $C_1>0$.  
Define $\tilde\mu_{\tau_2}=\e(\widetilde X_{\tau_2})$, and  
$\tilde\sigma^2_{\tau_2}=\text{Var}(\widetilde X_{\tau_2})$.
By the same approach used above, the Berry-Esseen theorem yields (suppressing the details)
$
\p\left(\widetilde X_{\tau_2}\in[-8,7]\right)=
\Theta\left(\frac{1}{\sqrt{n}}\right).
$
Combined with Eq. (\ref{eq:cov1}), this yields the desired result.
\end{proof}
We now combine Lemmas \ref{lem:lowerboundonp} and \ref{lem:covbnd} to bound $\text{Var}(X_\tau)$ where $\tau$ is a within-block transposition.
\begin{lemma}
\label{lem:varbnd}
With notation as above and assumptions as in Theorem \ref{thm:no_match}, let $\tau$ be a within-block transposition.  We have that
\begin{align*}\text{Var}(X_\tau)=O\left(\frac{N^2}{\sqrt{n}}  \right)
\end{align*}
\end{lemma}
\begin{proof}
From Eq. (\ref{eq:cov1}) there exists a constant $C_1>0$ such that
\begin{align*}
\text{Var}(\mathds{1}\{X_\tau\leq -1\})=(1-\p(X_\tau\leq -1))\p(X_\tau\leq -1)\leq \frac{1}{2}\text{exp}\left\{-C_1\rho^2 n\right\}
 +\Theta\left( \frac{1}{\sqrt{n}}\right).  
\end{align*}
Combined with Lemma \ref{lem:covbnd}, 
\begin{align*}
\text{Var}(X)\leq \frac{N}{2}e^{-C_1\rho^2 n}+\Theta\left( \frac{N}{\sqrt{n}}\right)+
N^2C_2\left(e^{-C_1\rho^2 n}+\Theta\left( \frac{1}{\sqrt{n}}\right)\right)
\Theta\left( \frac{1}{\sqrt{n}}\right)
= O\left(\frac{N^2}{\sqrt{n}}  \right)
\end{align*}
as desired.
\end{proof}


\subsection{Proof of Theorem \ref{thm3}}
\label{sec:recover}
The key to the proof of Theorem \ref{thm3} is the following consequence of Theorem \ref{them:GMSBM}:  If ($G_1,G_2)\sim\rho-$correlated SBM($K,\vec{n},b,\Lambda$) with $\rho=\omega(\sqrt{\log n/n})$ and respective adjacency matrices $A$ and $B$, then $\text{argmin}_{P\in \Pi(n)}\|AP-PB\|_F=\{I_n\}$ with high probability.
\begin{proposition}
Let $(G_1,\boldsymbol{\sigma}(G_2))\sim\boldsymbol{\sigma},\rho-$correlated SBM$(K,\vec{n},b,\Lambda)$ with $\rho=\omega(\sqrt{\log n/n})$.
Under the assumptions of Theorem \ref{them:GMSBM}, we have that
$\p\big[\GMs\neq G_2\big]=O(e^{-3\log n}).$
\end{proposition}
%
%
\begin{proof}
To ease notation, define
\begin{align*}
\p(x,y,z)&=\p\big[(G_1,G_2,\GMs)=(x,y,z)\big],\\ 
\p(x,z)&=\p\big[(G_1,\GMs)=(x,z)\big],\\
\p(x,y)&=\p\big[(G_1,G_2)=(x,y)\big].
\end{align*}
Note that
\begin{align*}
\p\big[\GMs\neq G_2\big]&=
\sum_{\substack{(x,y,z)\\\text{s.t. }z\neq y}}\p(x,y,z)=
\sum_{\substack{(x,y,z)\\ \text{s.t. }z\neq y}}\p(x,y)
\frac{\mathds{1}\{z\in P^*_{x,y}(y)\}}{|P^*_{x,y}(y)|}\\
&=\sum_{\substack{(x,y,z)\text{ s.t. }z\neq y,\\z\in P^*_{x,y}(y) }}\frac{\p(x,y)}{|P^*_{x,y}(y)|}
=\sum_{x,y}\sum_{\substack{z\in P^*_{x,y}(y)\\\text{s.t. }z\neq y}}\frac{\p(x,y)}{|P^*_{x,y}(y)|}
\leq \sum_{\substack{x,y \text{ s.t. }\\P^*_{x,y}(y)\neq\{y\}}}\p(x,y).
\end{align*}
As 
$P^*_{x,y}(y)\neq\{y\}\Rightarrow P^*_{x,y}\neq \{I\},$
this implies
$$\sum_{\substack{x,y \text{ s.t. }\\P^*_{x,y}(y)\neq\{y\}}}\p(x,y)\leq \sum_{\substack{x,y\text{ s.t.}\\P^*_{x,y}\neq \{I\}}}\p(x,y)=O(e^{-3\log n}),$$
where the last equality is an immediate consequence of Theorem \ref{them:GMSBM}.
\end{proof}
\noindent Theorem \ref{thm3} is then a straightforward application of Fano's inequality, which yields that 
\begin{align*}
&H\big[\GMs|G_2\big]=o(1),\text{ and }H\big[G_2|\GMs\big]=o(1).
\end{align*}
By the chain rule for entropy, it follows immediately that 
$H\big[G_2|G_1\big]=H\big[\GMs|G_1\big]+o(1),$ and  
$H\big[\GMs\big]=H\big[ G_2\big]+o(1).$
Combined, this yields that 
$I(G_1;G_2)-I(G_1;\GMs)=o(1),$
as desired.

\bibliographystyle{plain}
\bibliography{biblio}
\end{document}